\documentclass[letterpaper, 10 pt, conference, nofonttune]{ieeeconf} 

\IEEEoverridecommandlockouts                              

\usepackage{amsthm}

\newtheorem{theorem}{Theorem}

\usepackage{multirow}
\usepackage{subfigure}
\usepackage{subfloat}
\usepackage{color}
\usepackage{times}
\usepackage{epsfig}
\usepackage{amsmath}
\usepackage{amssymb}
\usepackage{mathrsfs}
\usepackage[ruled,vlined,linesnumbered]{algorithm2e}
\usepackage{graphicx} 
\usepackage{microtype}
\usepackage{wrapfig}
\newcommand{\RRR}{\mathbb{R}}
\newcommand{\GGG}{\mathcal{G}}
\newcommand{\RU}{\mathcal{R}}
\newcommand{\NNN}{\mathcal{N}}

\newcommand{\EEE}{\mathcal{E}}
\newcommand{\XXX}{\mathcal{X}}
\newcommand{\ZZZ}{\mathcal{Z}}
\newcommand{\LLL}{\mathcal{L}}
\newcommand{\MMM}{\mathcal{M}}

\newcommand{\PPP}{\mathcal{P}}

\newcommand{\FF}{\mathbf{F}}
\newcommand{\PP}{\mathbf{P}}

\newcommand{\HH}{\mathbf{H}}

\newcommand{\QQ}{\mathbf{Q}}
\newcommand{\KK}{\mathbf{K}}
\newcommand{\II}{\mathbf{I}}
\newcommand{\RR}{\mathbf{R}}
\newcommand{\WW}{\mathbf{W}}
\newcommand{\EE}{\mathbf{E}}

\newcommand{\MM}{\mathbf{M}}

\newcommand{\vv}{\mathbf{v}}
\newcommand{\xx}{\mathbf{x}}
\newcommand{\yy}{\mathbf{y}}
\newcommand{\zz}{\mathbf{z}}
\newcommand{\mm}{\mathbf{m}}
\newcommand{\hh}{\mathbf{h}}
\newcommand{\st}{\mathbf{s}}

\newcommand{\ee}{\mathbf{e}}

\definecolor{byzantine}{rgb}{0.74, 0.2, 0.64}

\title{\LARGE \bf
Multi-view Sensor Fusion by Integrating Model-based Estimation and Graph Learning for Collaborative Object Localization
}

\author{
}
\begin{document}
\author{Peng Gao$^{1}$, Rui Guo$^{2}$, Hongsheng Lu$^{2}$ and Hao Zhang$^{1}$
\thanks{$^{1}$Peng Gao and Hao Zhang are with the Human-Centered Robotics Laboratory at the Colorado School
of Mines, Golden, CO 80401, USA. $\{$gaopeng, hzhang$\}$@mines.edu}%
\thanks{$^{2}$Rui Guo and Hongsheng Lu are with Toyota Motor North America, Mountain View, CA 94043. $\{$rui.guo, hongsheng.lu$\}$@toyota.com}%
}

\maketitle
\thispagestyle{empty}
\pagestyle{empty}



\begin{abstract}
    Collaborative object localization aims to collaboratively estimate locations of objects observed from multiple views or perspectives, which is a critical ability for multi-agent systems such as connected vehicles. 
    To enable collaborative localization, several model-based state estimation and learning-based localization methods have been developed. 
    Given their encouraging performance,  model-based state estimation often lacks the ability to model the complex relationships among  multiple objects, 
    while learning-based methods are typically not able to fuse the observations from an arbitrary number of views
    and cannot well model uncertainty.
    In this paper, we introduce a novel spatiotemporal graph filter approach that 
    integrates graph learning and model-based estimation to perform multi-view sensor fusion for collaborative object localization. 
    Our approach models complex object relationships using a new spatiotemporal graph representation 
    and fuses multi-view observations in a Bayesian fashion to improve location estimation under uncertainty. 
    We evaluate our approach in the applications of connected autonomous driving and multiple pedestrian localization. Experimental results show that our approach outperforms previous techniques and achieves the state-of-the-art performance on collaborative localization.
\end{abstract}

\section{Introduction}



Object localization has been widely researched over the past decades due to its importance for situational awareness,
e.g., to understand street objects in autonomous driving.
The objective of object localization is to estimate the locations or positions of the objects in the environment
based on observations from robots and/or cameras. 
Most object localization methods focus on a single view,
in which objects are localized from a single perspective using observations obtained from a single robot or camera. 
Single-view localization methods have been used in various application domains, 
including object detection \cite{chen2016monocular,qi2018frustum} and tracking  \cite{sindagi2019mvx,zhang2019robust}, 
simultaneous localization and mapping \cite{bowman2017probabilistic,sharma2018beyond}, and scene reconstruction \cite{delmerico2018comparison,vasquez2014view}.

Recently, collaborative object localization based on multiple views has attracted increasing attention due to its reliability, scalability, and resilience to failures \cite{wei2018survey}.
The multiple views can be obtained either by multiple robots \cite{wasik2020robust} (e.g., connected vehicles) 
or multiple cameras \cite{xu2016multi}.
Multi-view collaborative localization outperforms single-view localization methods  
by integrating observations from multiple different perspectives to promote  shared situational awareness of the environment.
For example, as shown in Figure \ref{fig:motivation}, by integrating the measurements of the observed objects between two connected vehicles, 
the object locations can be better estimated.
Besides locating street objects by connected vehicles \cite{guo2019collaborative,marvasti2020cooperative},
multi-view collaborative localization is also applied to various real-world applications,
such as search and rescue \cite{acevedo2020dynamic,wang2018master} and manufacture \cite{dogar2019multi,gao2020visual}.

\begin{figure}[t]
\vspace{6pt}
\centering
\includegraphics[width=0.48\textwidth]{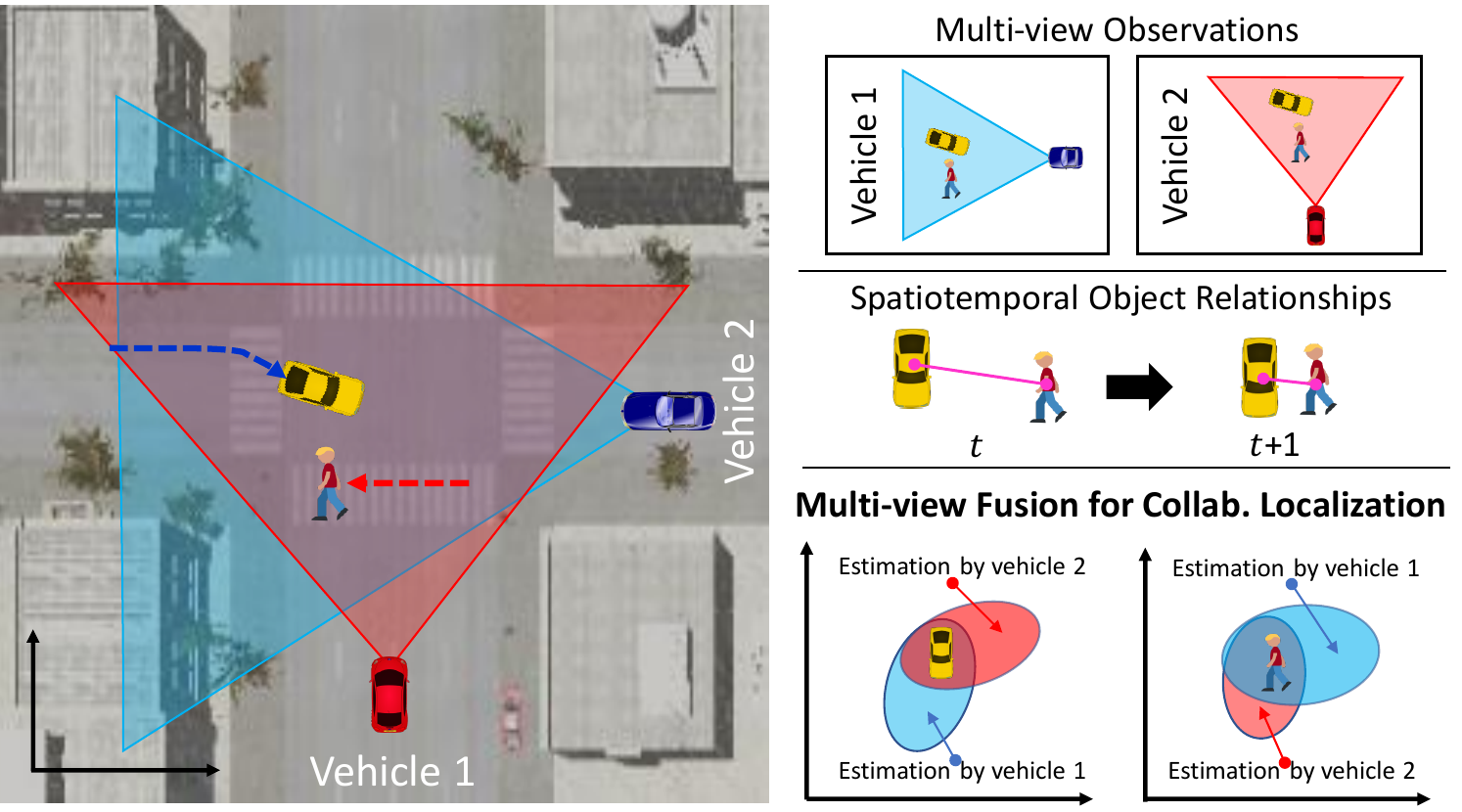}
\caption{
A motivating example of multi-view sensor fusion for  collaborative object localization in connected driving scenarios. When multiple connected vehicles localize multiple street objects, by encoding the objects' spatiotemporal cues for state estimation and integrating multi-view measurements, the final collaborative object localization results can be greatly improved.}
\vspace{-6pt}
\label{fig:motivation}
\end{figure}

Given its importance,
several methods have been developed to address collaborative localization of objects,
which can be broadly divided into two groups: model-based and learning-based methods. 
Model-based methods integrate state estimations and measurements under the Bayesian framework, e.g., using Kalman filters and extensions \cite{Weng2020_AB3DMOT,wang2018constrained}, particle filters \cite{deng2019poserbpf}, and Bayesian sequential filters \cite{stenger2006model,ullah2017hierarchical} to estimate the object locations.
Learning-based approaches typically employ neural networks to fuse multiple sensing modalities for location estimation,
e.g., based upon recurrent neural networks to model object motions \cite{altche2017lstm, zhang2019sr}, 
graph neural networks to encode the spatial relationship among objects \cite{li2019grip,alahi2016social},
and mixture architectures to integrate multiple relationships \cite{huang2019stgat, ivanovic2019trajectron}.


Both categories of the methods have their own advantages for object localization,
each category has its own shortcomings that have not been well addressed.
The shortcoming in model-based methods is caused by the inability to model the complex spatiotemporal relationship of the objects. 
For example, when estimating the location of a street object,
the historical information of the nearby objects is also important.
On the other hand, 
learning-based methods using deep networks often do not explicitly model uncertainty,
and lacks interpretability on the estimation procedure as well as flexibility 
to integrate an arbitrary number of inputs (e.g., neural networks often require a predefined and fixed number of views).

In this paper, we introduce a novel spatiotemporal graph filtering approach that integrates graph learning and model-based estimation in a principled fashion to perform multi-view sensor fusion for collaborative object localization.
For each view, we represent each observation as a graph, 
with the nodes to encode the locations of the detected objects
and the edges to encode the spatial relationship of the objects.
We also explicitly model the uncertainty in the observations of object locations.
To explicitly model the time dimension,
we represent a history of observations obtained from a view as spatiotemporal graphs.
When multiple views are available,
using their spatiotemporal graph representations,
we formulate collaborative localization as a multi-view sensor fusion problem.
Our method integrates spatiotemporal graph learning (that models the spatiotemporal relationship of the objects)
and model-based state estimation (that estimates locations in a Bayesion fashion and explicitly models uncertainty in both observations and estimations)
to address the formulated multi-view sensor fusion problem.

Our key contribution is the proposal of the novel spatiotemporal graph filter method 
that integrates spatiotemporal graph learning and model-based estimation to address collaborative object localization. 
Specifically, two novelties are proposed: 

\begin{itemize}
    \item We introduce a new representation for state estimation based on spatiotemporal graph neural networks,
which is able to not only encode complex spatiotemporal relationships of the objects
but also be readily integrated with model-based state estimation.
    \item We develop two new multi-view fusion gains that allow our method to fuse measurements from an arbitrary number of views, thus significantly improving the flexibility of our method in real-world applications when a dynamically changing number of views are available.
\end{itemize}

\section{Related Work}


\subsection{Multi-view Collaborative Object Localization}
Many methods have been designed to address multi-view collaborative object localization. One group of these methods is focused on using strict geometric constraints among multiple views to localize objects, e.g., based on strict camera parameters for stereo vision \cite{yao2018mvsnet,li2018stereo} and 3D reconstruction \cite{ji2017surfacenet,dong20174d}, ground plane assumptions for 3D object detection \cite{wu2019accurate}, and human geometric information for pedestrian tracking \cite{xu2016multi}. However, it is hard to obtain strict geometric information in multi-robot systems when the cameras mounted on robots are dynamic. Recently, the other group of the methods based on noisy geometric constraints are proposed. In these methods, the noisy geometric relationships obtained through GPS \cite{brahmbhatt2018geometry} or deep learning algorithm \cite{kendall2015posenet,yin2018geonet} are used to transform multi-view observations into the same coordination. The noisy geometric relationships are used for various applications, e.g., global mapping \cite{roth2003real}, object detection \cite{wasik2020robust} and scenes registration \cite{li2018scene}. 
Our approach allows to use noisy geometric constraints for collaborative object localization, which does not require fixed positions of cameras.

\subsection{Sensor Fusion}
Sensor fusion has been widely used to improve localization performance. We divide the methods into two groups, including model-based and learning-based methods. 
Model-based methods integrate state estimations and measurements under the Bayesian framework, e.g., using Kalman filter \cite{Weng2020_AB3DMOT}, extend Kalman filter \cite{wang2018constrained,panzieri2006multirobot}, particle filter \cite{deng2019poserbpf,qin2019surgical} and Bayesian sequential filters \cite{stenger2006model,ullah2017hierarchical} to estimate the object locations given their measured velocity, acceleration and location information. In addition, learning-based methods typically use neural networks to fuse multiple sensing modalities for location estimation, e.g., based on recurrent neural networks to model object motions \cite{altche2017lstm, zhang2019sr}, graph neural networks to encode spatial relationships among objects \cite{li2019grip,alahi2016social}, and mixture architectures to integrate multiple modalities, such as spatiotemporal relationships \cite{huang2019stgat, ivanovic2019trajectron}, and  visual-spatial relationships of objects \cite{yin2018geonet,meng2019signet,chen2017multi}.

For collaborative object localization, model-based methods are not able to model the complex spatiotemporal relationship of objects. In addition, learning-based methods are not able to model uncertainty and requires a predefined and fixed number of views as inputs. Our proposed approach that integrates both learning-based object localization and model-based state estimation in a principled fashion can well address these shortcomings in collaborative object localization.

\section{Preliminary}

Linear quadratic estimation, including Kalman filters and the extensions \cite{kleeman1992optimal},
is a widely used method for state estimation,
which estimates the states of a target based on noisy measurements.
Mathematically, linear quadratic estimation based on the Kalman filter is defined as:
\begin{equation} \label{eq:kf-1}
    \xx^t=\FF \xx^{t-1}
\end{equation}
\begin{equation}\label{eq:kf-2}
    \PP^{t}=\FF \PP^{t-1} \FF^{\top} + \QQ
\end{equation}
where $\xx^t$ denotes the state estimation of the target at time step $t$,
$\PP^t$ is the uncertainty of the state estimation,
$\FF$ is the state transition matrix that maps the state from time $t-1$ to $t$,
and $\QQ$ is the process uncertainty that  represents the uncertainty of the state estimation process.
The state $\xx^t$ and the uncertainty $\PP^t$ follow a multivariate Gaussian distribution $\NNN(\xx^t,\PP^t)$.
In addition to the state estimation,
this method can update the state to $\xx'^t$ from $\xx^t$ given a measurement $\zz^t$ at time $t$ as follows:
\begin{equation}\label{eq:kf-3}
    \KK^t=\PP^t \HH^{t \top} (\HH^t\PP^t\HH^{t \top} + \RR)^{-1}
\end{equation}
\begin{equation}\label{eq:kf-4}
    \xx'^t=\xx^t+\KK^t(\zz^t-\HH^t\xx^t)
\end{equation}
where $\RR$ denotes the uncertainty in the measurement $\zz^t$,
$\HH^t$ is the measurement transition matrix that maps the state to the measurement,
and $\KK^t$ is the Kalman gain that encodes the relative weight
of the uncertainty in both the state estimation and the measurement.
The measurement $\zz^t$ and the uncertainty of the measurement $\RR$ are also assumed to follow a multivariate Gaussian distribution $\NNN(\zz^t,\RR^t)$.
Then the uncertainty in the updated state $\xx'^t$ is computed by:
\begin{equation}\label{eq:kf-5}
    \PP'^t=(\II-\KK^t\HH^t) \PP^t
\end{equation}
where $\PP'^t$ is the uncertainty in the updated state.
It has been approved that this method converges to the optimal solution for single-view state estimation \cite{kleeman1992optimal}.

\section{Approach}

\textbf{Notation.}
We use superscript $t$ and $v$ to represent the time step and the view index, respectively.
We also use subscript $i$ to denote the index of the object in a view.
For example, $\vv^{t,v}_i$ denotes a feature vector of the $i$-th object observed in the $v$-th view at time $t$.

\subsection{Problem Formulation}
Given $n$ views (e.g., from $n$ robots and/or cameras) that are  partially overlapped,
we denote the observations recorded in all $n$ views as
$\MMM=\{\MMM^v\},v=1,2,\dots,n$,
where $\MMM^v=\{\GGG^{1,v},\GGG^{2,v},\dots,\GGG^{t,v}\}$ is a sequence of observations obtained in the $v$-th view from time $1$ to $t$.
Each observation $\GGG^{t,v}$ is represented as a graph
$\GGG^{t,v}=\{\ZZZ^{t,v},\RU^{t,v},\EEE^{t,v}\}$.
The node set $\ZZZ^{t,v}=\{\zz_{1}^{t,v},\zz_{2}^{t,v},\dots,\zz_{m}^{t,v}\}$ denotes the measurements (or observations) of all the objects' 3D locations,
where $\zz_{i}^{t,v} \in \RRR^{3}$ denotes the measurement of the $i$-th object's location observed in the $v$-th view at time $t$,
and $m$ is the number of observed objects.
The uncertainty set $\RU^{t,v}=\{\RR_{1}^{t,v},\RR_{2}^{t,v},\dots,\RR_{m}^{t,v}\}$ denotes the uncertainty in $\ZZZ^{t,v}$,
where $\RR_i^{t,v} \in \RRR^{3 \times 3}$ is defined as the covariance of $\zz_i^{t,v}$.
Then, $\zz_{i}^{t,v}$ and $\RR_i^{t,v}$ can be assumed to follow a multivariate Gaussian distribution $\NNN(\zz_i^{t,v},\RR_i^{t,v})$ with $\zz_i^{t,v}$ as the mean and $\RR_i^{t,v}$ as the covariance.
The edge set $\EEE^{t,v}=\{e^{t,v}_{i,j}\}$ denotes the spatial relationships between a pair of objects, where $\ee^{t,v}_{i,j}=1$, if $\zz_i^{t,v}$ and $\zz_j^{t,v}$ are connected.


We represent the locations of the objects as the states $\XXX^t=\{\XXX^{t,v}\}, v=1,,2,\dots,n$,
where $\XXX^{t,v}$ denotes the states that are estimated from the $v$-th view  at time $t$.
$\XXX^{t,v}=\{\xx_i^{t,v}\},i=1,2,\dots m$, 
contains the state estimations of all the $m$ objects observed in the $v$-th view, 
where $\xx_i^{t,v}$ is the estimated states of the $i$-th object in the $v$-th view at time $t$.
In addition, 
we use $\PPP^t=\{\PPP^{t,v}\},v=1,2,\dots,n$ to denote the uncertainties in $\XXX^t$.
$\PPP^{t,v}=\{\PP_i^{t,v}\},i=1,2,\dots,m$, includes the uncertainties in $\XXX^{t,v}$, 
and $\PP_i^{t,v} \in \RRR^{3\times 3}$ denotes the state uncertainty in $\xx_i^{t,v}$. 
Then, $\xx_i^{t,v}$ and $\PP_i^{t,v}$ can be assumed to follow the multivariate Gaussian distribution $\NNN(\xx_i^{t,v},\PP_i^{t,v})$.



The problem that is addressed in this paper focuses on multi-view state estimation for collaborative object localization
in order to estimate the states (i.e., 3D locations) $\XXX^t$ of multiple objects 
by integrating observations from multiple views $\MMM=\{\MMM^v\}$.  
In order to address the challenge that model-based state estimation methods often cannot well encode 
the spatiotemporal relationship among objects,
our approach constructs a representation based on deep graph learning to model this complex spatiotemporal relationship (Section \ref{sec:learning}).
In order to integrate an arbitrary number of views and improve robustness to noise,
which has not been well addressed by learning-based estimation,
we integrate learning-based and model-based state estimation in a principled way 
to perform collaborative object localization (Section \ref{sec:modeling}).

\color{black}

\subsection{Encoding Spatiotemporal Relationships of the Objects} \label{sec:learning}
Our first novelty is the design of a learning-based method to model the complex spatiotemporal relationship among objects for the state estimation of objects, which is defined as follows:
\begin{equation}\label{eq:formulation_state_est}
    \xx_i^{t,v}=\Phi\left(\GGG^{1,v},\GGG^{2,v},\dots,\GGG^{t-1,v}\right)
\end{equation}
where $\Phi$ denotes the spatiotemporal graph learning network that estimates the object's state $\xx_i^{t,v}$ given its spatiotemporal information encoded in the spatiotemporal graphs $\{\GGG^{1,v},\GGG^{2,v},\dots,\GGG^{t,v}\} \in \MMM^v$. 
The graph learning network $\Phi$ consists of three components, including a LSTM encoder, a graph attention neural network and a LSTM decoder.

First, we design a LSTM-based module to encode the temporal motion of each object as follows:
\begin{equation}
 \Delta \zz^{t-1,v}_{i}=\zz^{t-1,v}_{i}-\zz^{t-2,v}_{i}
\end{equation}
\begin{equation}\label{eq:lstm-encode}
    \mm^{t-1,v}_{i}=\phi\left(\mm^{t-2,v}_{i}, \Delta \zz^{t-1,v}_{i},\WW^e\right)
\end{equation}
where $\Delta \zz^{t-1,v}_{i}$ denotes the relative location (motion) of the $i$-th object in the $v$-th view from time $t-2$ to $t-1$, $\phi$ denotes a one-layer LSTM encoder network with the trainable parameter matrix $\WW^e$ that is shared among all objects, and $\mm^{t-1,v}_{i}$ denotes the temporal motion embedding of the $i$-th object, which captures the motion of the object.

Second, we introduce a graph attention neural network to capture the spatial impacts of each object from its surrounding objects as follows:
\begin{align}\label{eq:Gat}
    &\alpha_{i,j}^{t-1,v}=  \\
    &\frac{\exp \left(\text{ReLu} (\WW^a \mm^{t-1,v}_{i}||\WW^a \mm^{t-1,v}_{j})\right)}{\sum_{e^{t-1,v}_{i,j}}\exp \left(\text{ReLu} (\WW^a \mm^{t-1,v}_{i}||\WW^a \mm^{t-1,v}_{j})\right)} \nonumber
\end{align}
where $||$ denotes the concatenation operation, $\WW^a$ is a trainable parameter matrix that is shared by all the objects, $e^{t-1,v}_{i,j}$ denotes the connection between the $i$-th object and its $j$-th neighborhood object, and $\alpha_{i,j}^{t-1,v}$ denotes the impact of the $j$-th neighborhood object to the $i$-th object in the $v$-th view at time $t-1$. By using the SoftMax function with ReLu nonlinear function, we normalize the impacts of the $i$-th object from its neighbors. Given the normalized spatial impacts, the spatial embedding of the impacts of each object from its surrounding objects are defined as follows:
\begin{equation}\label{eq:Gat_update}
    \st^{t-1,v}_{i}=\text{ReLu}\left(\sum_{e^{t-1,v}_{i,j}}\alpha_{i,j}^{t-1,v}\WW^a \mm^{t-1,v}_{j}\right)
\end{equation}
where $\st^{t-1,v}_{i}$ denotes the spatial embedding of the $i$-th objects, which captures the impacts of the $i$-th object from its neighbors.

Third, the object location at time $t$ is estimated through a LSTM decoder given its spatiotemporal embedding, which is defined as follows:
\begin{equation} \label{eq:spatiotemporal}
    \hh^{t,v}_{i}=\mm^{t,v}_{i}||\st^{t,v}_{i}
\end{equation}
where $\hh^{t,v}_{i}$ denotes the spatiotemporal embedding of the $i$-th object, which is the concatenation of its temporal embedding $\mm^{t,v}_{i}$ and spatial embedding $\st^{t,v}_{i}$. The object location is estimated as follows:
\begin{equation} \label{eq:lstm-decoder}
    \xx^{t,v}_{i}=\psi\left(\hh^{t-1,v}_{i}, \Delta \zz^{t-1,v}_{i}, \WW^d\right)
\end{equation}
where $\psi$ denotes a one-layer LSTM decoder network with a trainable parameter matrix $\WW^d$. The final state $\xx^{t,v}_{i}$ is estimated given the relative location $\Delta \zz^{t-1,v}_{i}$ and the spatiotemporal embedding $\hh^{t-1,v}_{i}$. The uncertainty of the state estimation $\xx^{t,v}_{i}$ is defined as follows:
\begin{equation}\label{eq:formulation_state_un}
    \PP_i^{t,v}= \PP_i^{t-1,v} + \QQ_i^{t,v}
\end{equation}
where $\QQ_i^{t,v}$ denotes the process uncertainty caused by the model bias in the sptiotemporal graph learning network $\Phi$. 
Square error loss is used to train the network $\Phi$: $\LLL=\sum_{v=1}^{n}\sum_{i=1}^{m}||\xx_i^{t,v}-\yy_i^{t,v}||_2$, 
where $\yy_j^{t,v}$ denotes the ground truth location of the $i$-th object in the $v$-th view at time $t$.

\subsection{Integrating Arbitrary Multi-view Measurements} \label{sec:modeling}
As the second novelty of this paper,
we introduce a multi-view state estimation method that 
integrates learning-based and model-based state estimation to fuse measurements from an arbitrary number of views. Given the learned state estimations obtained in Eq. (\ref{eq:lstm-decoder}) and the measurements from an arbitrary number of views,
our approach estimates the states of objects as follows:


\begin{equation}\label{eq:formulation_state_update}
    \hat{\xx}_i^{t,v}=g(\xx_i^{t,v},\zz_i^{t,1},\zz_i^{t,2},\dots,\zz_i^{t,v})
\end{equation}
where $g$ denotes the state update function, and $\hat{\xx}_i^{t,v}$ denotes the updated state estimation obtained by integrating the learned state estimation $\xx_i^{t,v}$ and the multi-view measurements $\zz_i^{t,v},v=1,2,\dots,n$. 
Specifically, the state update function $g$ includes two new multi-view gains, which are defined as follows:
\begin{equation}\label{eq:estimation gain}
    \EE_i^{t,v}=\left((\PP_i^{t,v})^{-1}+\sum_{v=1}^{n} (\RR_i^{t,v})^{-1})\right)^{-1}(\PP_i^{t,v})^{-1}
\end{equation}
\begin{equation}\label{eq:measurement gain}
    \MM_i^{t,v}=\left((\PP_i^{t,v})^{-1}+\sum_{v=1}^{n} (\RR_i^{t,v})^{-1})\right)^{-1}(\RR_i^{t,v})^{-1}
\end{equation}
where $\EE_i^{t,v} \in \RRR^{3\times 3}$ denotes the state estimation gain for the learned state estimation $\xx_i^{t,v}$, which describes the relative weight of the learned state estimation uncertainty $\PP_i^{t,v}$ in the total uncertainty, and $\MM_i^{t,v} \in \RRR^{3\times 3}$ denotes the measurement gain for the measurement $\zz_i^{t,v}$, which describes the relative weight of the measurement uncertainty $\RR_i^{t,v}$ in the total uncertainty.
Intuitively, if the measurement uncertainty $\RR_i^{t,v}$ becomes greater, then the measurement gain $\MM_i^{t,v}$ becomes smaller, which means that we trust the learned state estimation more than the measurement. In addition, $\EE_i^{t,v}$ and $\MM_i^{t,v}$ follow the constraint $\sum_{v=1}^{n} \MM_i^{t,v}+\EE_j^{t,v}=\mathbf{I}$, where $\mathbf{I} \in \RRR^{3\times 3}$ denotes an identical matrix.  
Given $\MM_i^{t,v}$ and $\EE_j^{t,v}$, the state update function $g$ is defined as follows:
\begin{equation} \label{eq:update_state}
    \hat{\xx}_i^{t,v}= \EE_i^{t,v}\xx_i^{t,v} +\sum_{v=1}^{n}\MM_i^{t,v}\zz_i^{t,v}
\end{equation}
where the updated state estimation $\hat{\xx}_i^{t,v}$ is computed through the sum of the learned state estimation $\xx_i^{t,v}$ and multi-view measurements $\zz_i^{t,v},v=1,2,\dots,n$, which are weighted by the multi-view gains. 

The uncertainty of the updated state estimation is defined as follows:
\begin{equation} \label{eq:update_uncertainty}
    \hat{\PP}_i^{t,v}=\left(\left(\PP_i^{t,v}\right)^{-1}+ \sum_{v=1}^n (\RR_i^{t,v})^{-1}\right)^{-1}
\end{equation}
where $\hat{\PP}_i^{t,v}$ denotes the uncertainty of the updated state estimation $\hat{\xx}_i^{t,v}$, which is obtained by integrating both the learned state estimation uncertainty $\PP_i^{t,v}$ and measurement uncertainties $\RR_i^{t,v},v=1,2,\dots,n$. 

Our method of multi-view sensor fusion for collaborative object location is presented in Algorithm \ref{alg:OptAlgorithm}. 
It is worth noting that the updated state estimation and state uncertainty can be used as a part of the historical sequence to continuously improve state estimation in future iterations.
In addition, we prove that our multi-view fusion method is equivalent to linear quadratic estimation when only one view exists.

\begin{algorithm}[t]
	\SetAlgoLined
	\SetKwData{Left}{left}\SetKwData{This}{this}\SetKwData{Up}{up}
	\SetKwFunction{Union}{Union}\SetKwFunction{FindCompress}{FindCompress}
	\SetKwInOut{Input}{Input}\SetKwInOut{Output}{Output}
	\SetNlSty{textrm}{}{:}
	\SetKwComment{tcc}{//}{} 
	
	\small
	
	\Input{
	$\MMM=\{\MMM^v\}$ (multi-view measurements from time $1$ to $t$).
	}

	\Output{
		$\XXX^{t,v}$ (state estimations in the $v$-th view at time $t$)
	}
	\BlankLine

    \For {$i=1:m$}
	    {
	
	        Estimate the state $\xx_i^{t,v}$ by Eqs. (\ref{eq:lstm-encode}-\ref{eq:lstm-decoder});
	
	        Compute the state estimation uncertainty $\PP_i^{t,v}=\PP_i^{t-1,v} + \QQ_i^{t-1,v}$
	
	        Compute the state estimation gain $\EE_i^{t,v}$ by Eq. (\ref{eq:estimation gain});
		
	        Compute the measurement gain $\MM_i^{t,v}$ by Eq. (\ref{eq:measurement gain});
	
	        Compute the updated state estimation $\hat{\xx}_i^{t,v}$ by Eq. (\ref{eq:update_state});
		
	        Compute the updated state uncertainty $\hat{\PP}_i^{t,v}$ by Eq. (\ref{eq:update_uncertainty});
	
	        Update $\zz_i^{t,v}=\xx_i^{t,v}$ and $\RR_i^{t,v}=\PP_i^{t,v}$
	  }
	\Return $\XXX^{t,v}$

	\caption{The proposed spatiotemporal graph filter method for collaborative object localization.}\label{alg:OptAlgorithm}

\end{algorithm}

\begin{theorem}\label{theorem:1}
Assuming $\PP$ and $\RR$ are invertible, when $n=1$, $\MM$ defined in Eq. (\ref{eq:measurement gain}) is equivalent to $\KK$ defined in Eq. (\ref{eq:kf-3}).
\end{theorem}
\begin{proof}
When $n=1$, then Eq. (\ref{eq:measurement gain}) is equal to the following:
\begin{equation} \label{eq:proof1}
    \MM =\left(\PP^{-1}+\RR^{-1} \right)^{-1}\RR^{-1}
\end{equation}
then, we can convert Eq. (\ref{eq:proof1}) to the following:
\begin{equation} \label{eq:proof2}
     \MM =\PP(\PP +\RR)^{-1}\RR \RR^{-1} =\PP(\PP +\RR)^{-1}= \KK
\end{equation}


Since the measurement transition matrix $\HH \in \RR^{3 \times 3}$ defined in Eq. (\ref{eq:kf-3}) is an identical matrix in our case, when $n=1$, the state estimation gain $\MM$ defined in Eq. (\ref{eq:measurement gain}) is equivalent to the Kalman gain $\KK$ defined in Eq. (\ref{eq:kf-3}).
\end{proof}


\begin{theorem}\label{theorem:2}
Assuming $\PP$ and $\RR$ are invertible, when $n=1$, $\hat{\PP}$ defined in Eq. (\ref{eq:update_uncertainty}) is equivalent to $\PP'$ defined in Eq. (\ref{eq:kf-5}).
\end{theorem}
\begin{proof}
When $n=1$, Eq. (\ref{eq:update_uncertainty}) is equal to the following:
\begin{equation} \label{eq:proof3}
    \hat{\PP}=(\PP^{-1}+ \RR^{-1})^{-1}
\end{equation}
then, we can convert Eq. (\ref{eq:proof3}) to the following:
\begin{equation} \label{eq:proof4}
    \hat{\PP}=\RR(\PP+ \RR)^{-1}\PP = (\II-\KK)\PP =\PP'
\end{equation}
Thus, when $n=1$, $\hat{\PP}$ defined in Eq. (\ref{eq:update_uncertainty}) is equivalent to $\PP'$ defined in Eq. (\ref{eq:kf-5}). It is easy to prove that the updated state $\hat{\xx}$ in Eq. (\ref{eq:update_state}) is equivalent to the updated state $\xx'$ in Eq. (\ref{eq:kf-4}).
\end{proof}

\begin{figure}[t]
\centering

\subfigure[CAD]{
\centering
\includegraphics[height=2.2cm]{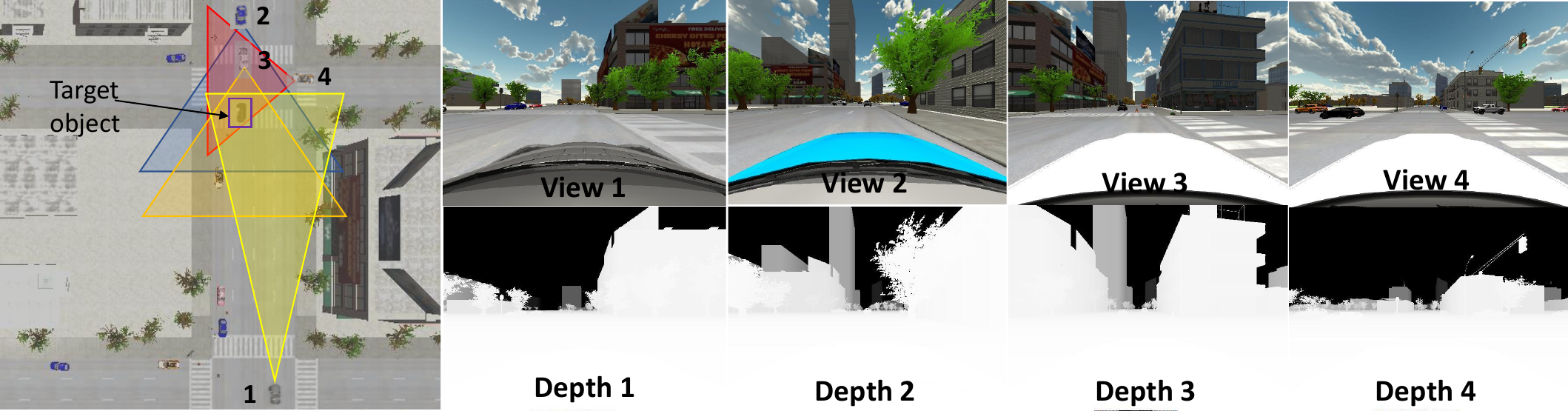}\label{fig:cad}
}\\
\subfigure[MPL]{
\centering
\includegraphics[height=2.2cm]{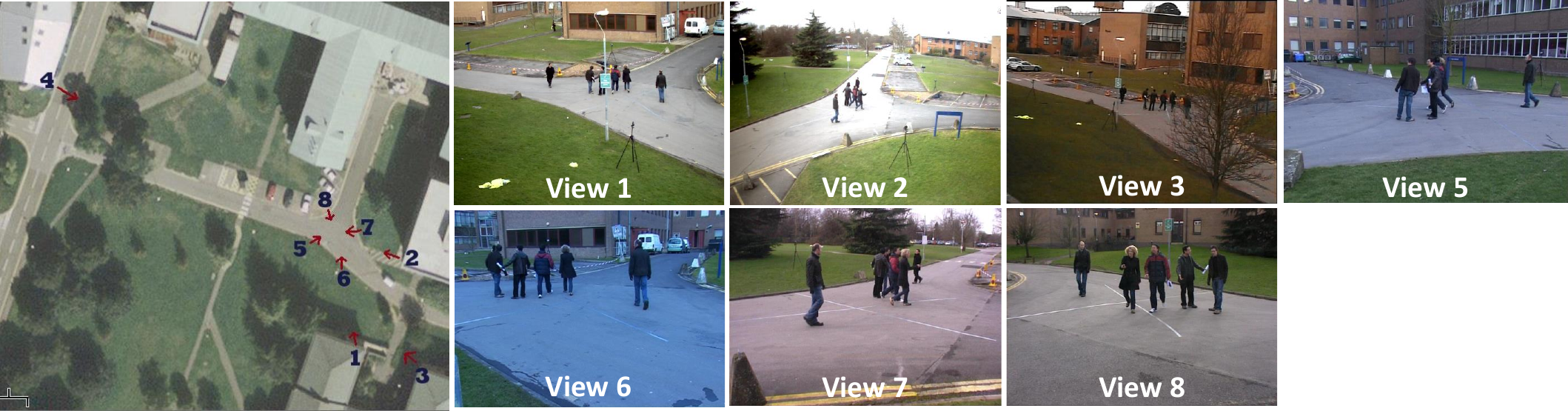}\label{fig:mpl}
}
\caption{Examples of CAD and MPL. We aim to estimate the location of the objects in CAD and MPL. In CAD, we have four observations acquired from different connected autonomous vehicles to collaboratively localize objects. In MPL, we have seven observations acquired from different cameras (view 1 to view 3 and view 5 to view 8) for collaborative object localization. The ground truth locations of objects are provided in the overhead view.
}\label{fig:dataset}
\vspace{-6pt}
\end{figure}

\section{Experiments}

\subsection{Experimental Setup}
We utilize both high-fidelity robotic simulations and real-world applications
to evaluate our method for collaborative localization in multi-view systems
in two scenarios, including
simulated connected autonomous driving (CAD) provided by Toyota as shown in Figure \ref{fig:cad},
and real-world multi-pedestrian localization (MPL) generated from PET 2009 \cite{ferryman2009pets2009} as shown in Figure \ref{fig:mpl}. 
\begin{itemize}
    \item CAD includes $3500$ data instances. Each data instance consists of $3 \times 4 = 12$ trajectories, which are generated through $3$ objects observed by $4$ connected vehicles. Each trajectory consists of at least $8$ measured locations recorded by a RGB-D camera at $10$Hz. The ground truth locations are provided by the connected vehicle simulator. We use $1500$, $500$ and $1500$ data instances for training, validation and testing, separately. 
    \item MPL includes $4500$ data instances. Each data instance consists of $4 \times 7=28$ trajectories, which are generated by $4$ pedestrians observed by $7$ cameras mounted at different positions. Each trajectory consists of at least $20$ locations recorded at $7$Hz. The ground truth 3D locations are obtained following the recent work \cite{xu2016multi}. We use $1500$, $1000$ and $2000$ data instances for training, validation and testing, separately. We assume that the same group of objects are measured by all the views and the correspondences of the objects among views are identified following the method \cite{gaoregularized}.
\end{itemize}

We assume that the same group of objects are measured by all the views and the correspondences of the objects among views are identified following the method \cite{gaoregularized}.

We construct each observation acquired by each view as a graph with node attributes generated by 3D locations. The edges are generated by fully connection given the 3D locations.
The LSTM-based encoder $\phi$ and decoder $\psi$ only contains one layer with $\WW^e$ having the dimension of $3 \times 32$ and $\WW^d$ having the dimension of $32 \times 3$. The dimension of the hidden state $\mm$ is set to $32$. The graph attention network consists of two layers, with $\WW^a$ set to $16 \times 16$ in the first layer and $16 \times 32$ in the second layer. Initially, the state $\xx$ is set to a all zero matrix, the state estimation uncertainty $\PP$ is set to a diagonal matrix with the diagonal values set to $10000$, the measurement uncertainty $\RR$ is  set to $1000$, and the process uncertainty $\QQ$ is set to $2000$. We use alternating direction method of multipliers (ADMM) as the optimization method in all experiments.


We implement our full approach using learning-based state estimation and the multi-view gains for multi-view measurements integration. We also implement a baseline method Ours-L, which only uses the learning-based state estimation for object localization given historical spatiotemporal observations. In addition, we compare with three previous methods for collaborative localization of objects, including:


\begin{itemize}
    \item Average of measurements for 3D reconstruction (\textbf{AOM}) \cite{ji2017surfacenet}, which averages the locations of the same objects in different views.
    \item Multi-modality fusion for multi-object tracking (\textbf{MMT}) \cite{Weng2020_AB3DMOT}, which uses model-based Kalman filter with a fixed velocity of objects as an assumption for state estimations.
    \item Spatiotemporal state estimation for pedestrian trajectory prediction (\textbf{STTP}) \cite{huang2019stgat}, which uses learning-based neural network to localize objects given a single-view spatiotemporal measurements.
\end{itemize}

Following the widely used experimental setup \cite{huang2019stgat,godard2017unsupervised}, we use displacement error (\textbf{DE}) to evaluate our approach, which is defined as the distance between the estimated location and the ground truth location to evaluate the localization accuracy. In addition, we also use relative displacement error (\textbf{REL-DE}) to evaluate our approach, which is defined as the ratio of the displacement error over the ground truth location, in order to evaluate the localization accuracy relative to the measurement distance.

\begin{table}
\centering
\tabcolsep=0.375cm
\vspace{6pt}
\caption{Quantitative results of our approach and comparisons with previous and baseline methods in CAD and MPL.
}
\label{tab:QuanResults}
\begin{tabular}{|l|c|c|c|c|}
\hline
Method & \multicolumn{2}{ c| }{CAD} & \multicolumn{2}{ c| }{MPL}  \\
\cline{2-5}
		 & DE & Rel-DE & DE  & Rel-DE\\
		\hline\hline
		AOM \cite{ji2017surfacenet}   &2.5449& 0.3067  &3.3672 &0.1623  \\
		\hline
		MMT \cite{Weng2020_AB3DMOT} &2.1210   & 0.2508    &2.8234 &0.1258   \\
		\hline
		STTP \cite{huang2019stgat} &1.3109   & 0.1512  &2.8652 &0.1357    \\
		\hline\hline
		{Ours-L}  &1.1562  &0.1421  &2.0350 &0.0976    \\
		\hline
		\textbf{Ours}  &\textbf{1.1387} & \textbf{0.1377} &\textbf{1.5656} & \textbf{0.0736}\\
		\hline	
\end{tabular}
\vspace{-10pts}
\end{table}

\begin{figure}[ht]
\centering
\subfigure[MMT]{
\centering
\includegraphics[width=0.155\textwidth]{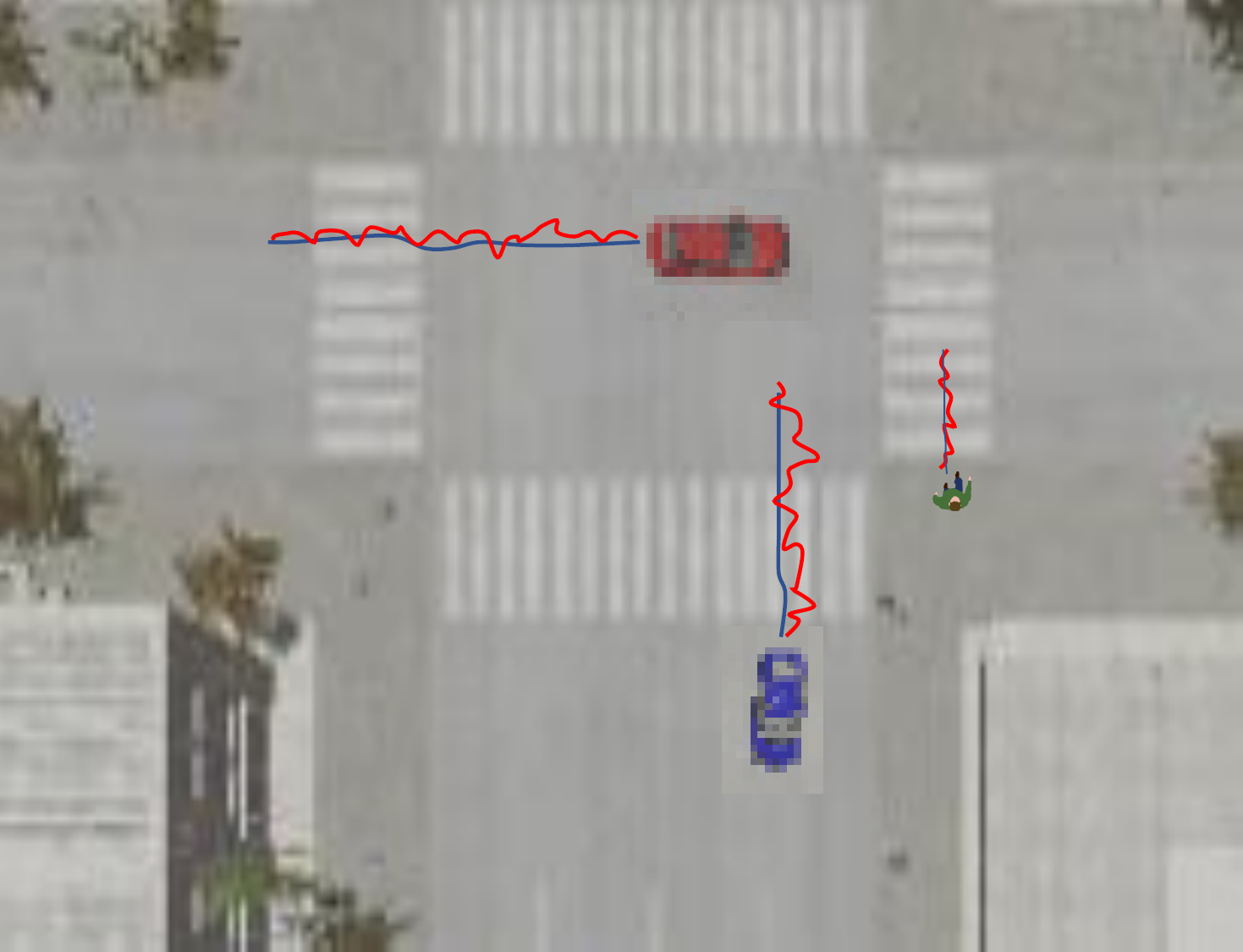}\label{fig:cad-mmt}
}
\hspace{-9pt}
\subfigure[STTP]{
\centering
\includegraphics[width=0.155\textwidth]{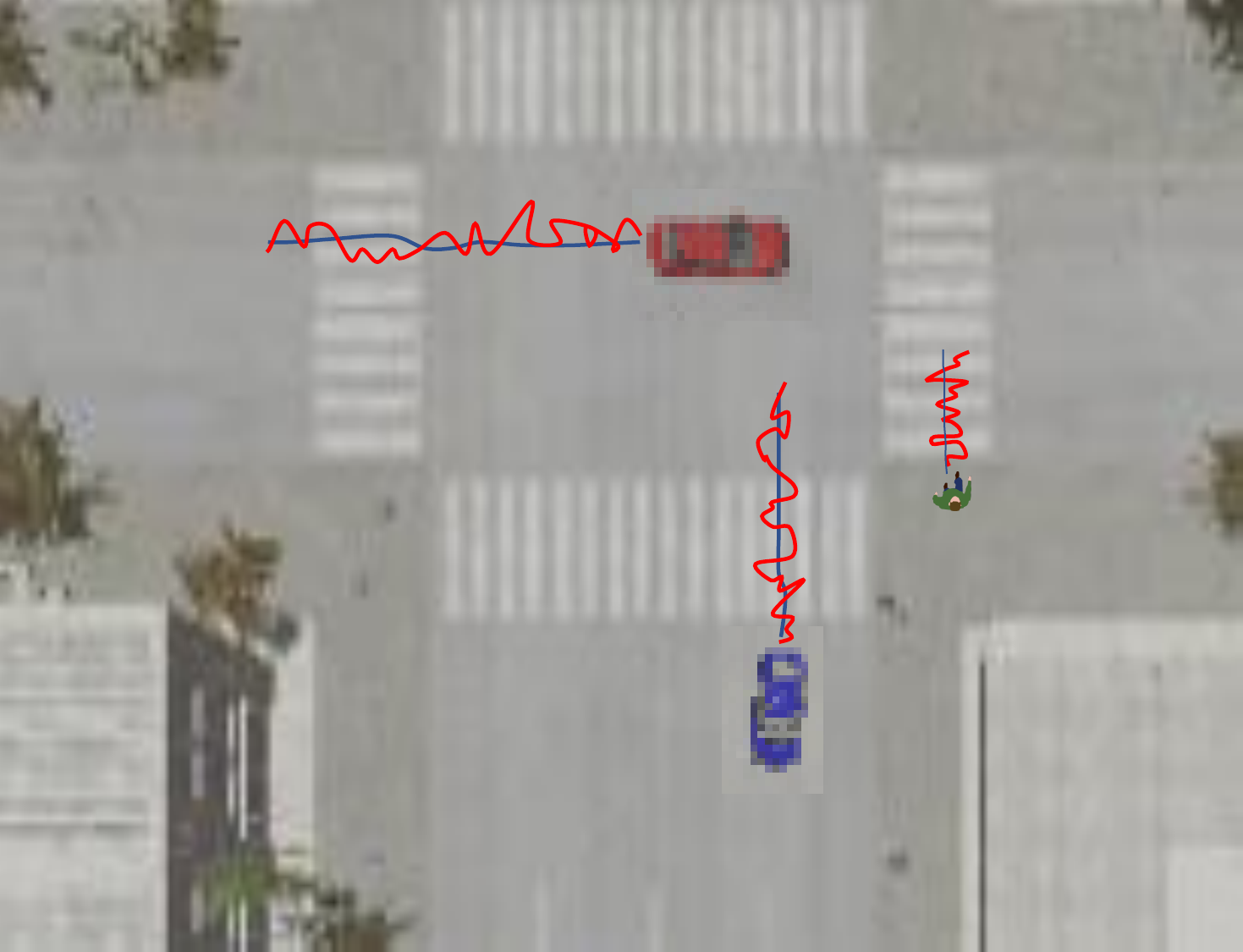}\label{fig:cad-sttp}
}
\hspace{-9pt}
\subfigure[{Our Approach}]{
\centering
\includegraphics[width=0.155\textwidth]{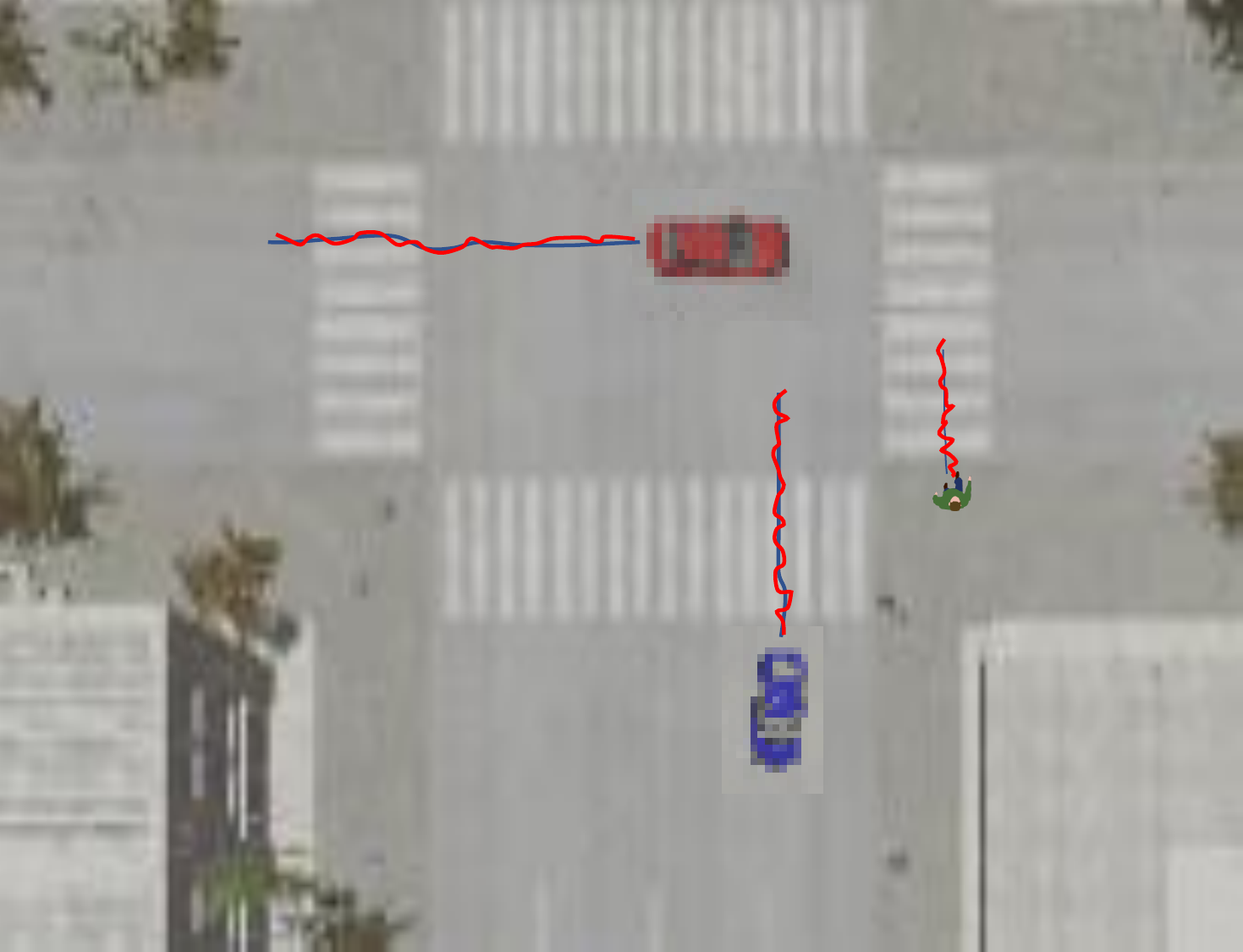}\label{fig:cad-ours}
}
\caption{Qualitative experimental results of our approach over CAD in overhead view and comparisons with the MMT and STTP methods. Blue lines denote ground truth trajectories and red lines denote the estimations. 
}
\vspace{-6pt}
\label{fig:QualResults-cad}
\end{figure}

\subsection{Connected Autonomous Driving Simulation}
Our approach is first evaluated in the CAD scenario. The object instances used in the simulation include dynamic pedestrians and vehicles observed by an arbitrary number of connected vehicles ranging from $1$ to $4$. Objects have complex spatiotemporal relationship among each other, e.g., vehicles stop at the intersection when pedestrians walk across the street. The measurements of objects include large noise.

The qualitative results in CAD are shown in Figure \ref{fig:QualResults-cad}.
It demonstrates that our approach works well on the localization of objects observed by multiple vehicles.
Comparisons with MMT and STTP are also presented in Figure \ref{fig:QualResults-cad}.
We observe that the linear estimation MMT does not work well due to the non-linear movements of objects. The performance of the single-view method STTP has large displacement error due to the noisy measurements in a single view.
Our method outperforms all these methods, which can model the complex spatiotemporal relationship among objects and integrate multi-view measurements to mitigate the state estimation uncertainty and to improve the performance in the simulations.

The quantitative results in CAD are presented in Table \ref{tab:QuanResults}. We can observe that our baseline method estimating states given multi-view fusion results outperforms the single-view method STTP. The results indicate the importance of integrating learning-based state estimation and model-based multi-view fusion for collaborative object localization.
AOM performs badly due to the noisy measurements. MMT performs better than AOM given the model-based state estimation. The learning-based method STTP can significantly improve the performance due to its capability of nonlinear modeling. Our approach achieves the best performance by modeling complex relationship of objects and integrating multiple measurements.

\begin{figure}[ht]
\centering
\subfigure[MMT]{
        \centering
        \includegraphics[width=0.16\textwidth]{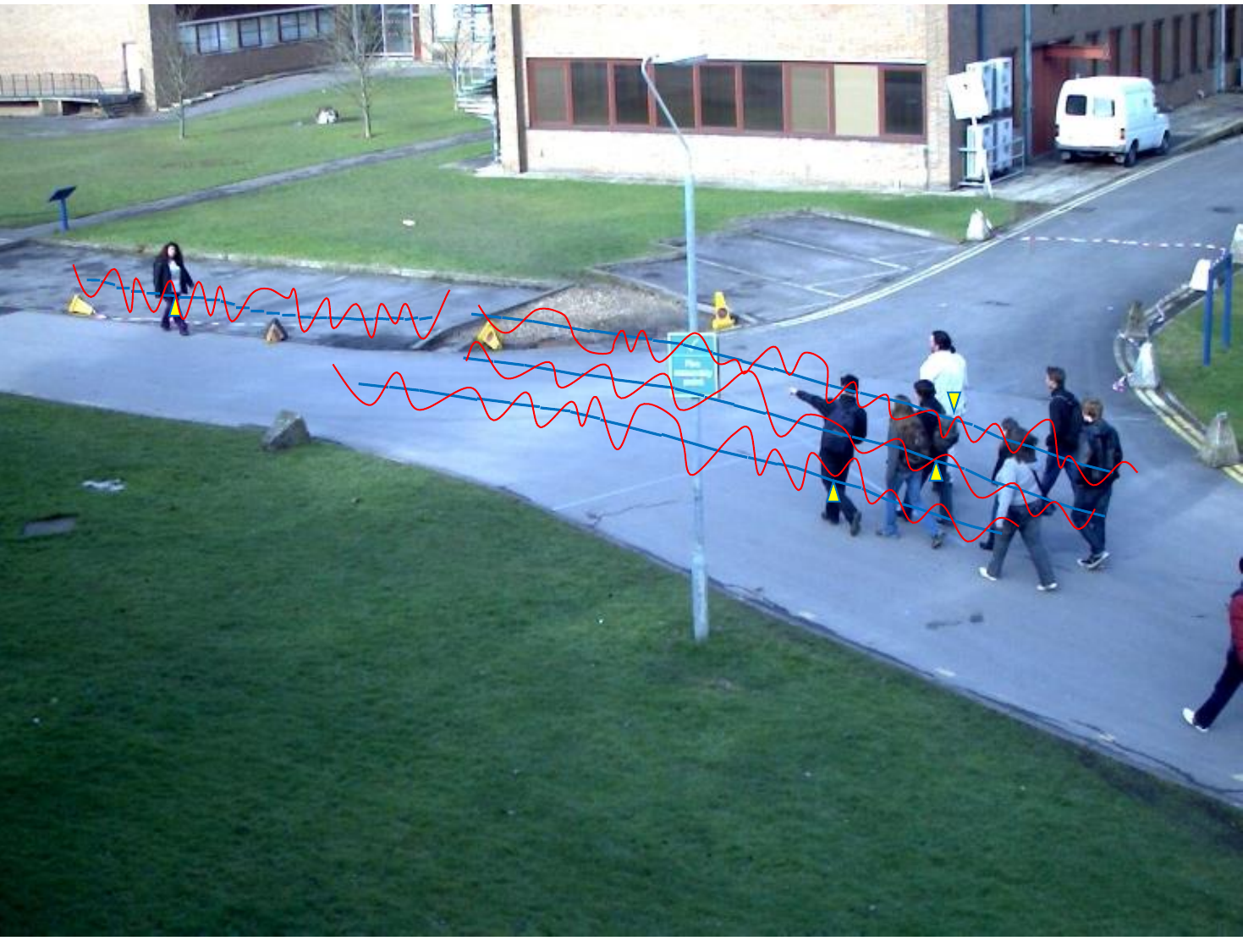}\label{fig:mpl-mmt1}
        \hspace{-4pt}
        \includegraphics[width=0.16\textwidth]{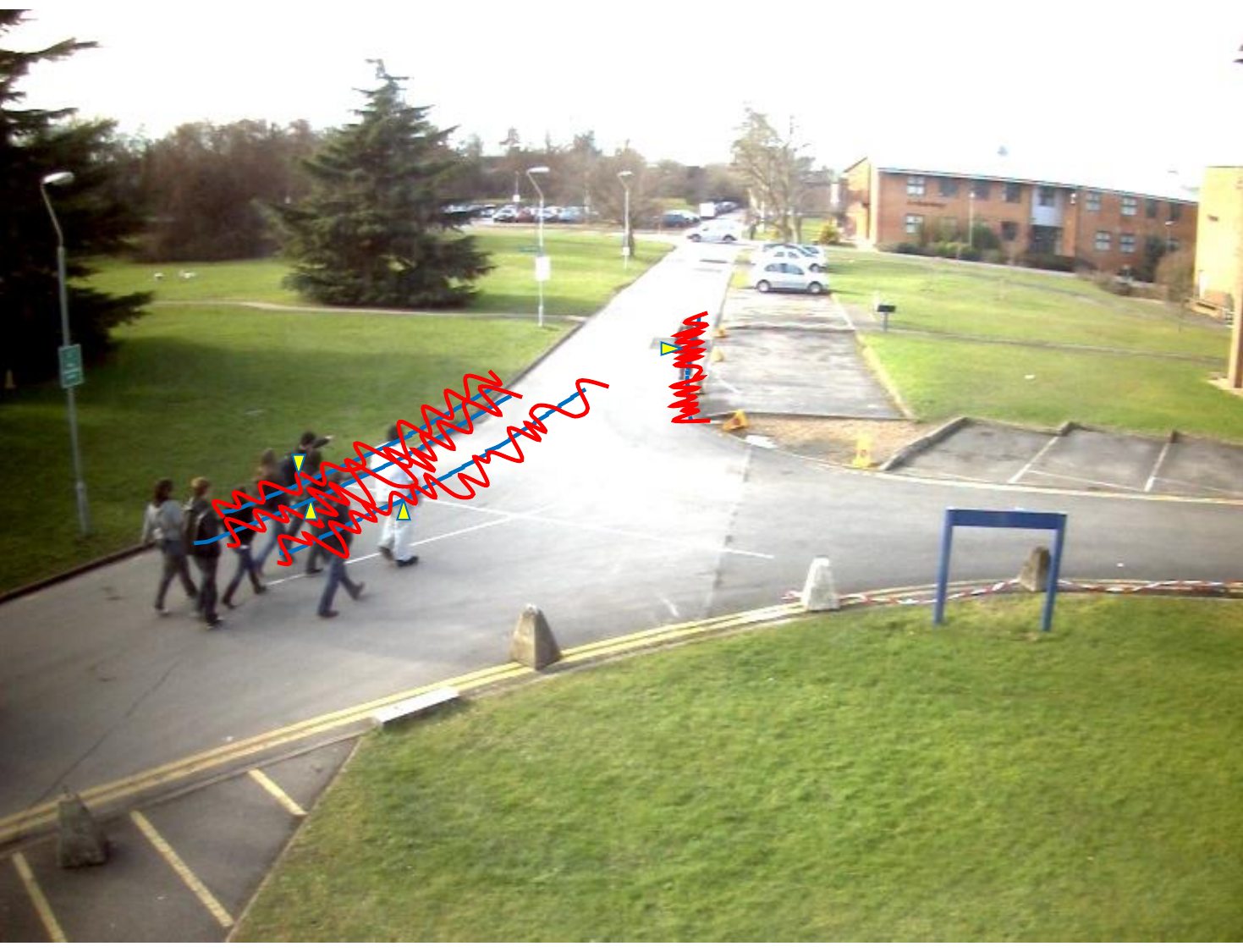}\label{fig:mpl-mmt2}
        \hspace{-4pt}
        \includegraphics[width=0.16\textwidth]{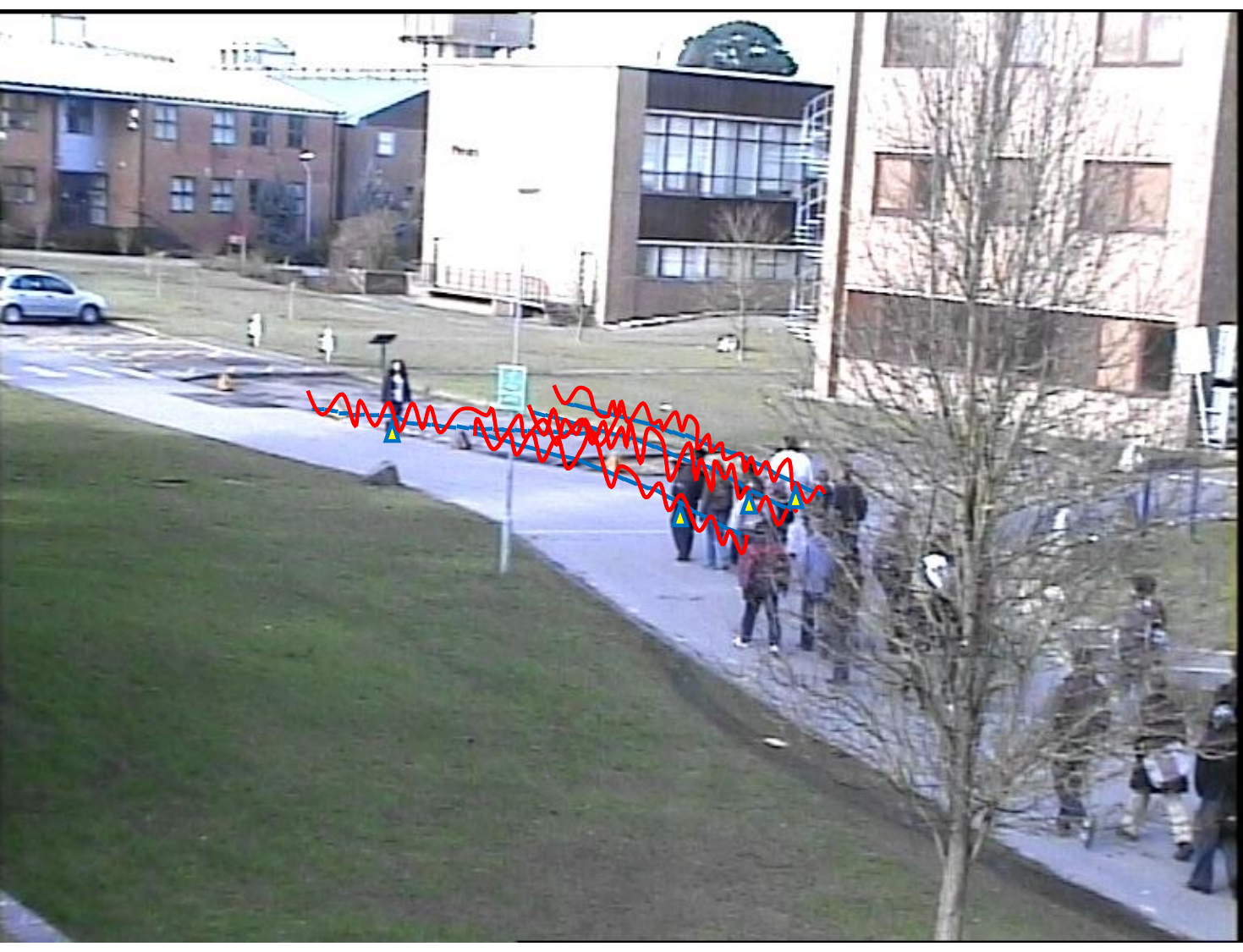}\label{fig:mpl-mmt3}
}\\
\subfigure[STTP]{
    \vspace{-30pt}
    \includegraphics[width=0.16\textwidth]{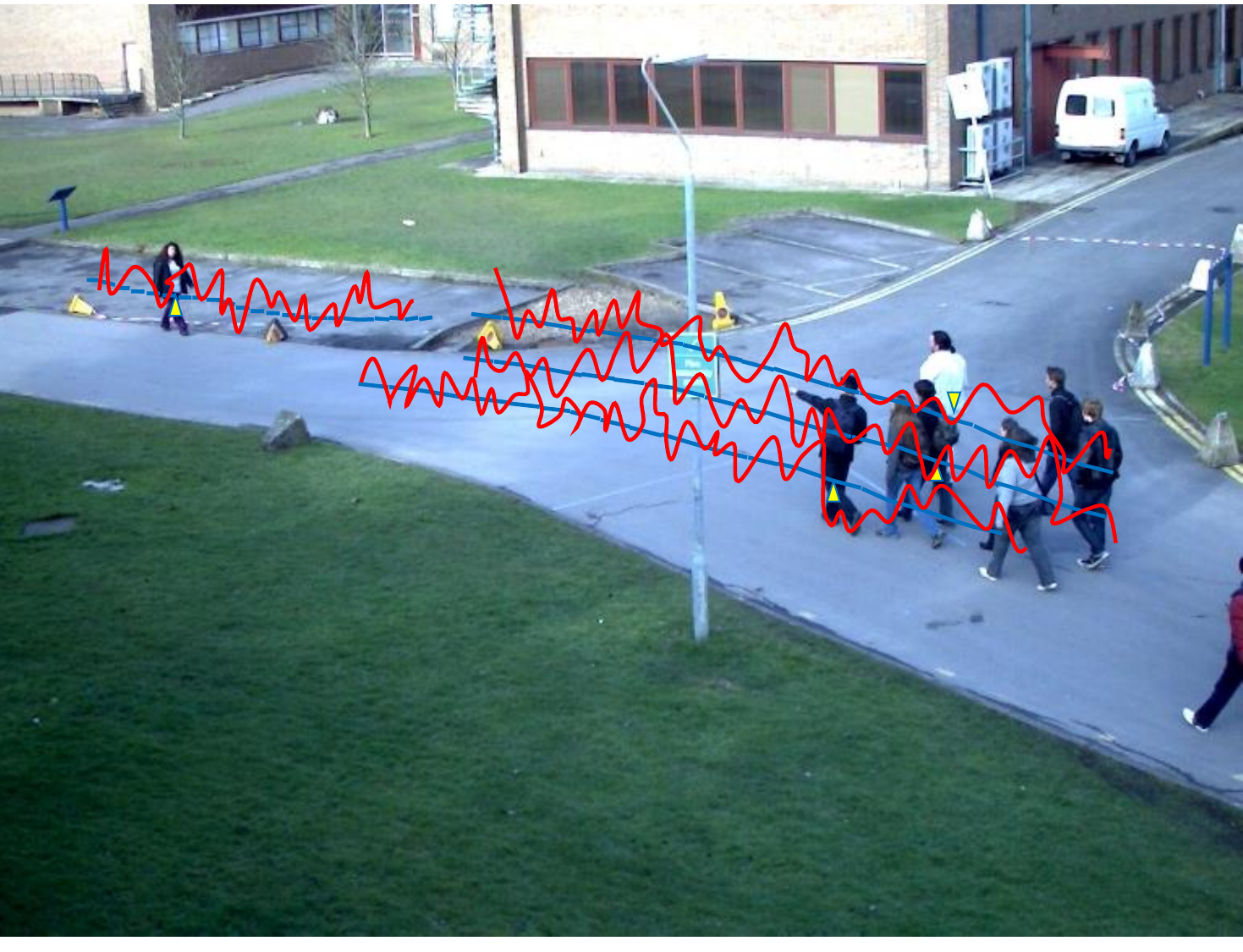}\label{fig:mpl-sttp1}
    \hspace{-4pt}
    \includegraphics[width=0.16\textwidth]{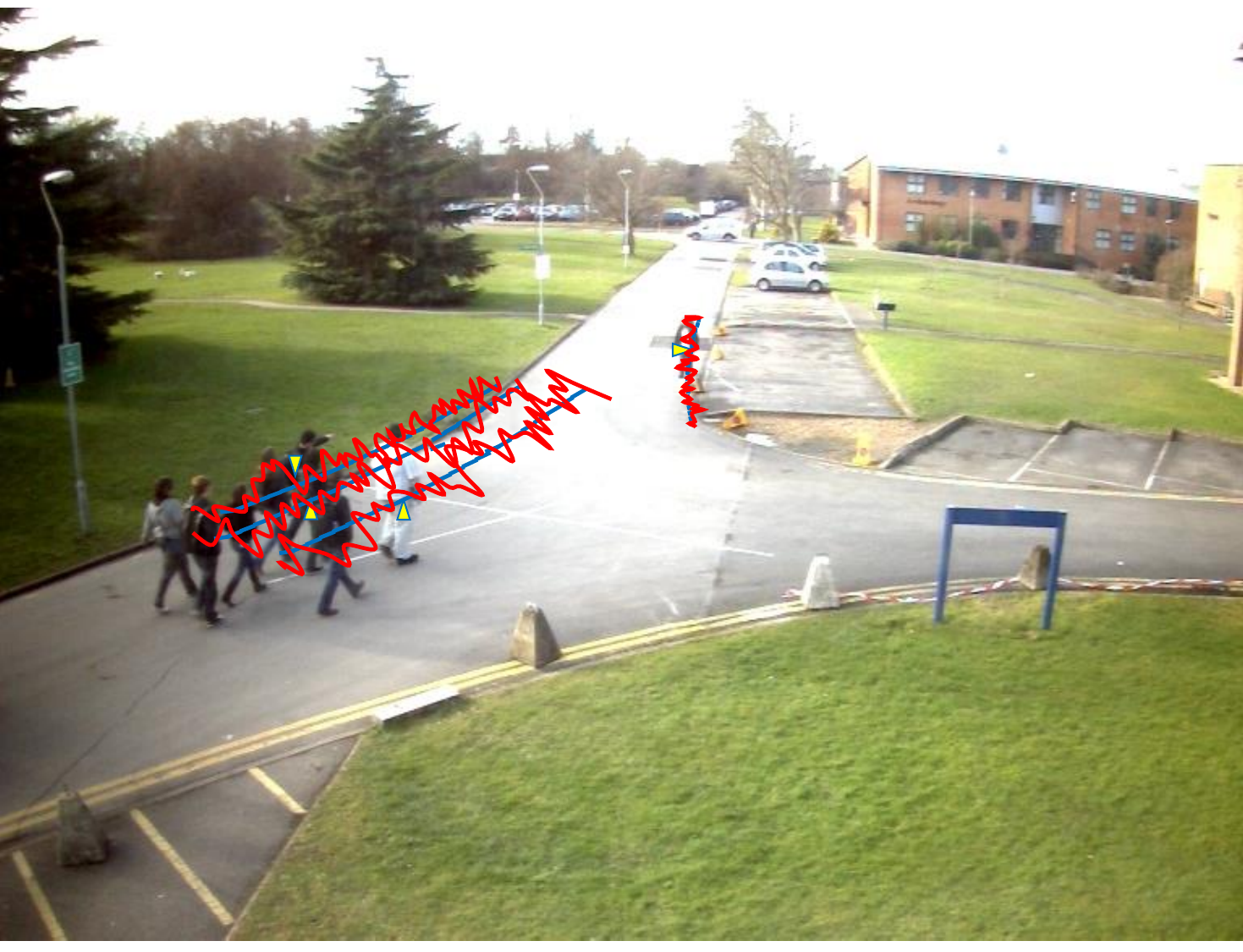}\label{fig:mpl-sttp2}
    \hspace{-4pt}
    \includegraphics[width=0.16\textwidth]{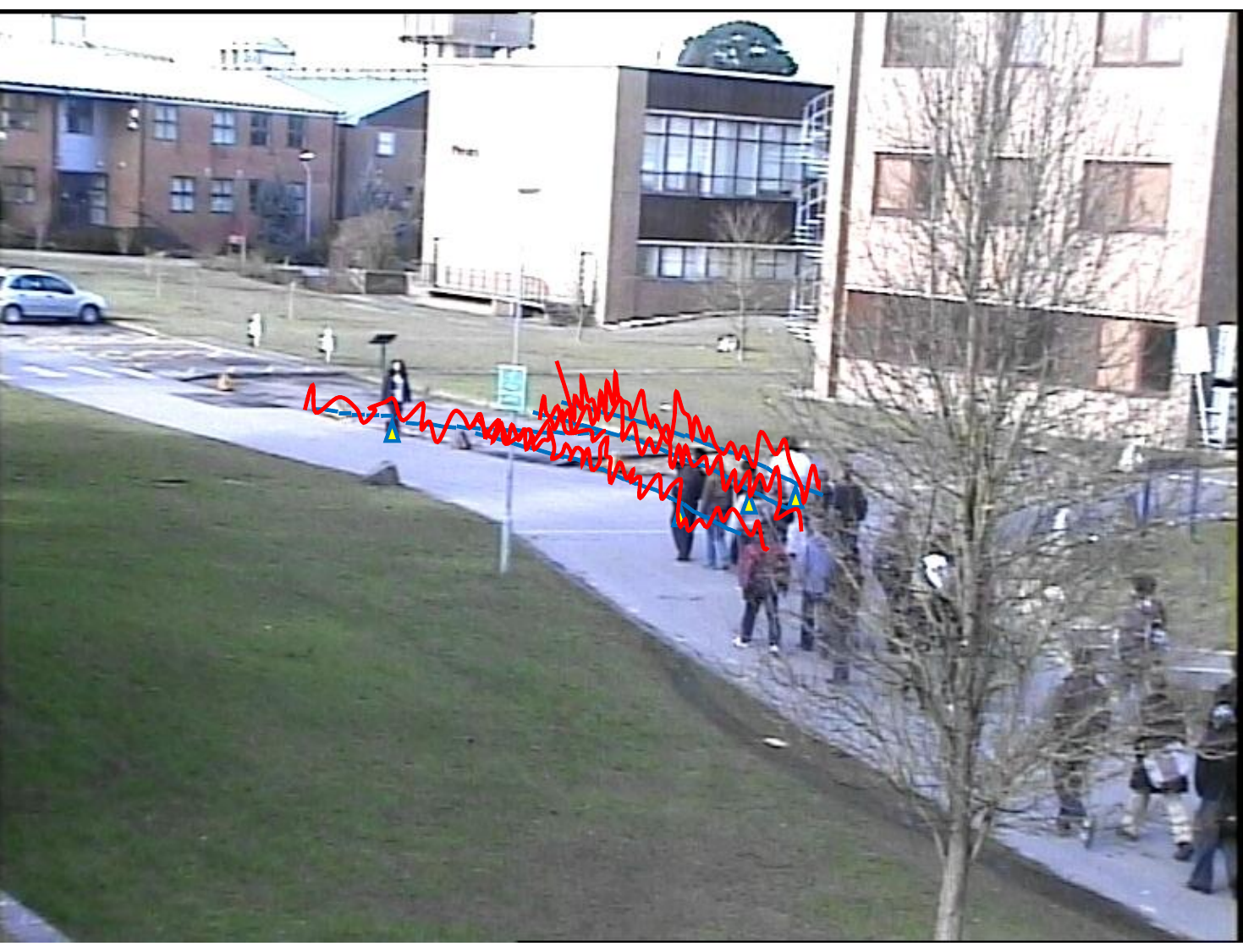}\label{fig:mpl-sttp3}
}\\
\subfigure[Ours]{
    \vspace{-30pt}
    \includegraphics[width=0.16\textwidth]{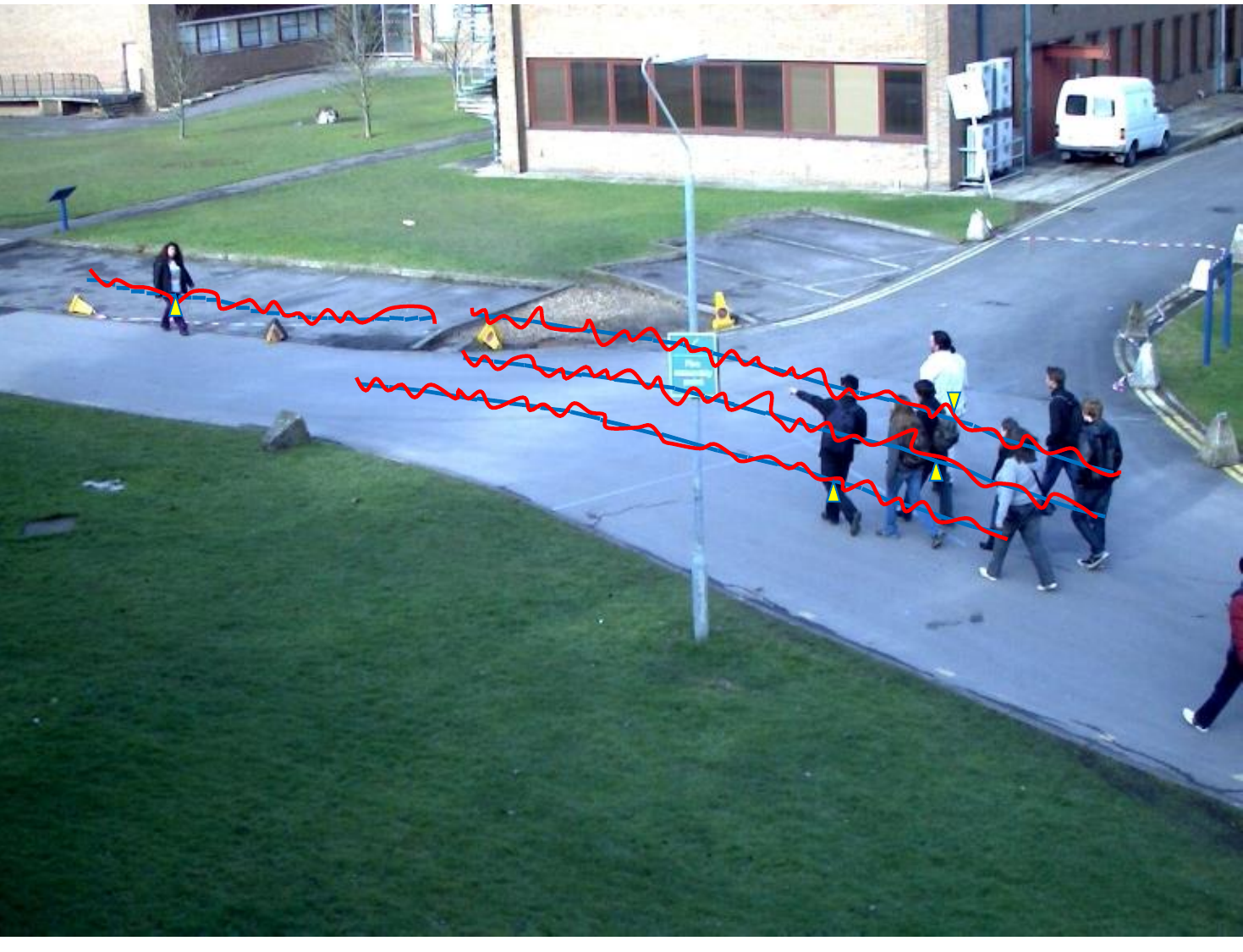}\label{fig:mpl-outs1}
    \hspace{-4pt}
    \includegraphics[width=0.16\textwidth]{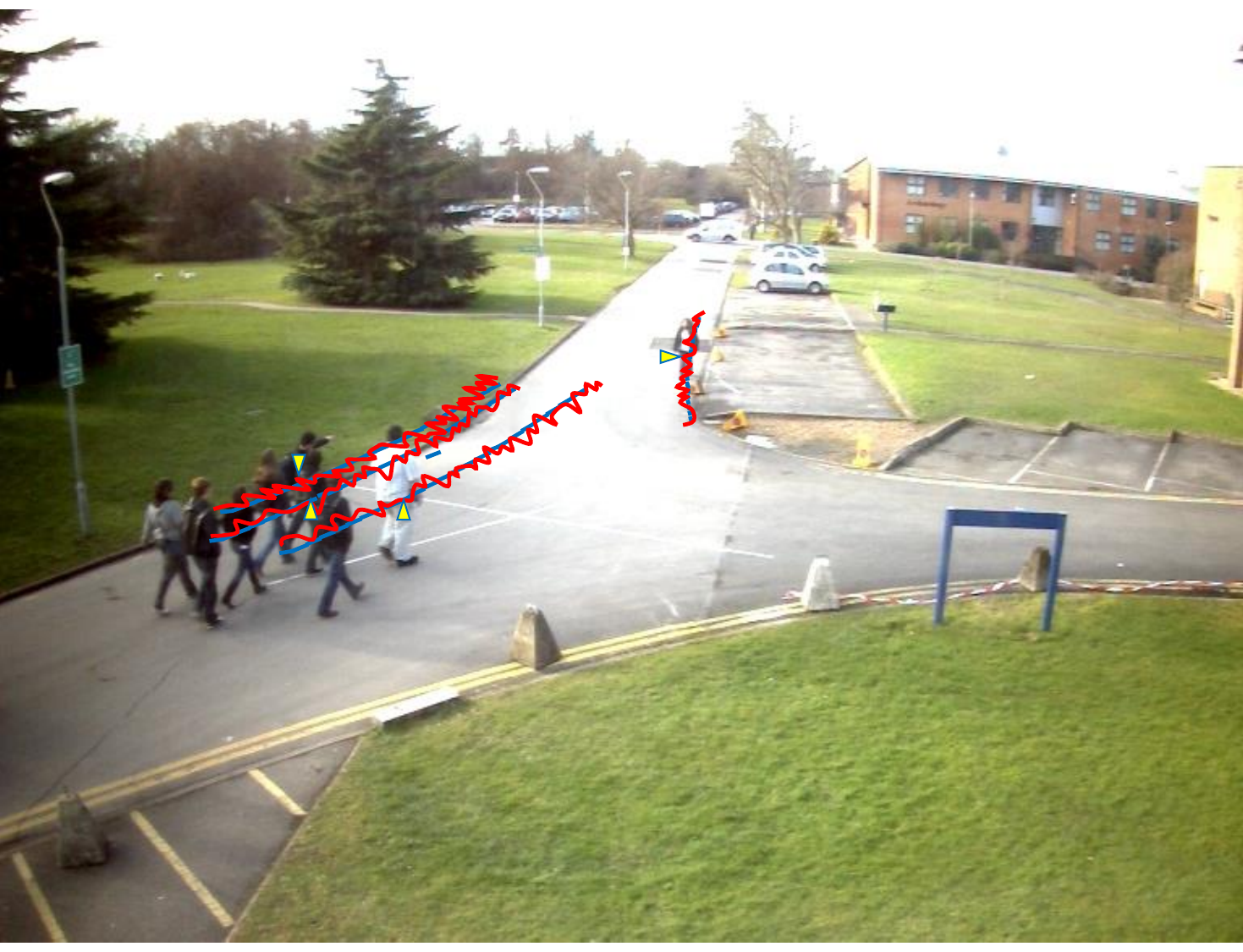}\label{fig:mpl-ours2}
    \hspace{-4pt}
    \includegraphics[width=0.16\textwidth]{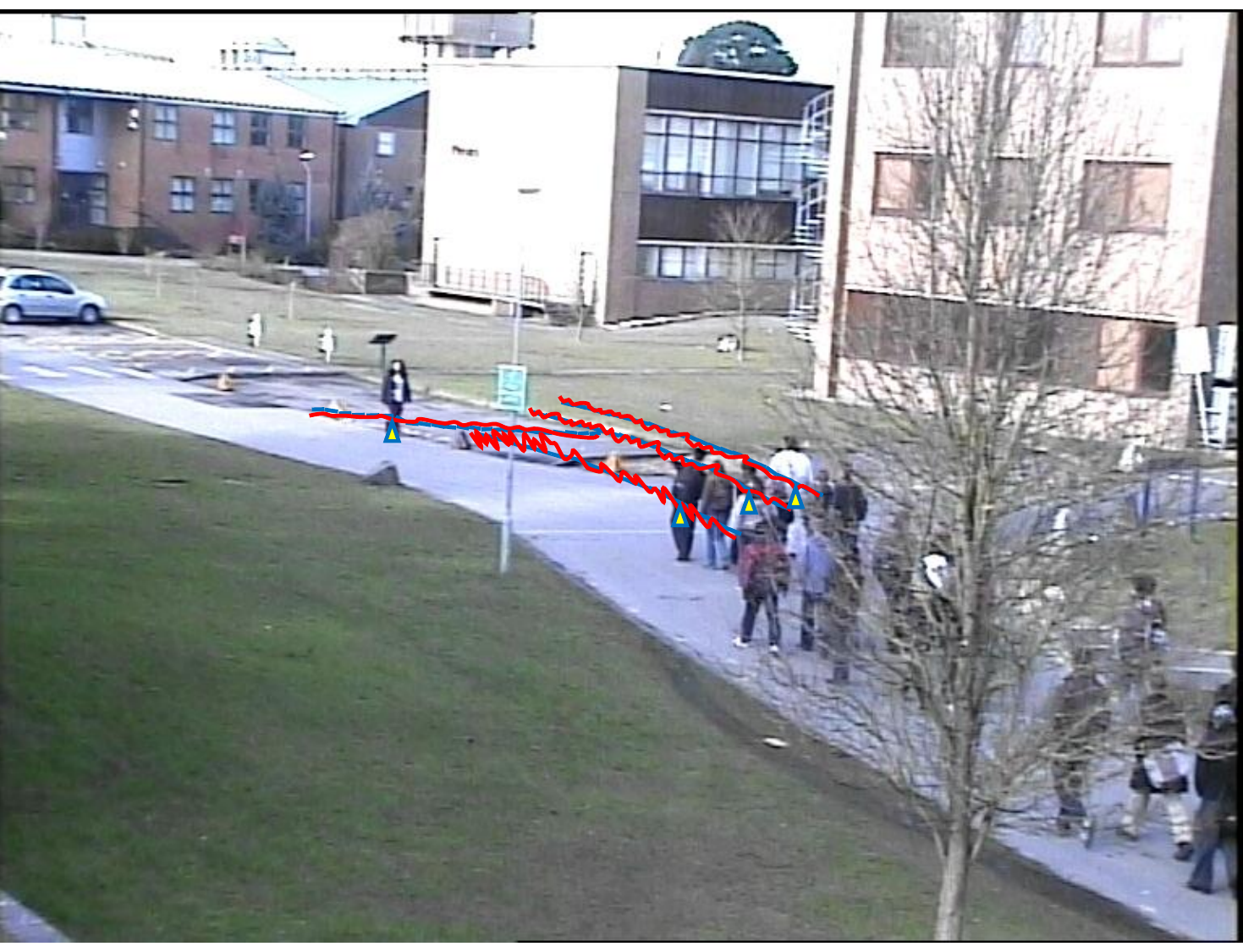}\label{fig:mpl-ours3}
}
\caption{Qualitative experimental results of our approach over MPL in view 1-3 and comparisons with the MMT and STTP methods. Blue lines denote ground truth trajectories and red lines denote the estimations. 
}\label{fig:QualResults-mpl}
\vspace{-6pt}
\end{figure}

\subsection{Real-world Multiple Pedestrian Localization}
We perform additional evaluation in the MPL scenario.
In MPL, the object instances are dynamic pedestrians, which have complex spatiotemporal relationships, e.g., changing moving direction to avoid collision. In addition, the measured locations of pedestrians include large noise.

The qualitative results in MPL are demonstrated in Figure \ref{fig:QualResults-mpl}.
The results indicate that MMT and STTP can localize objects but the performance is inaccurate and have large estimation uncertainty due to the noisy measurements.
By addressing both measurement uncertainties and the complex spatiotemporal relationship of objects in state estimations, our approach obtains the best results in  collaborative object localization.

The quantitative results in MPL are listed in Table \ref{tab:QuanResults}.
MMT and STTP obtain an improved result compared with AOM by partially addressing the challenges, e.g., integrating multi-view measurements and modeling spatiotemporal relationships of objects.
By addressing both of the challenges, our full approach obtains the best performance in the experiments of multiple pedestrians localization in the real world.


\begin{wrapfigure}{RI}{0.24\textwidth}
\centering
\vspace{-12pt}
\includegraphics[width=0.235\textwidth]{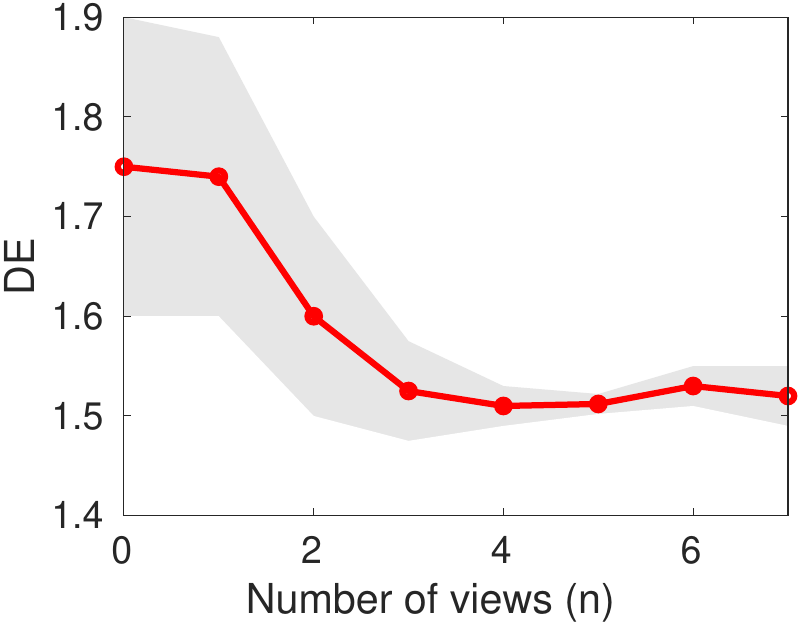}
\centering
\vspace{-12pt}
\caption{
Hyper-parameter analysis. }
\label{fig:param}
\vspace{-10pt}
\end{wrapfigure}

We also studied the influence of the number of collaborative views $n$ on MPL. The results from our approach that localizes objects given the different number of measurements provided by multi-view systems are shown in Figure \ref{fig:param}. We can see that the performance of our approach gradually improves (the decrease of DE and uncertainty visualized by the thickness of the dark area) as the increase of the number of  collaborative views. When $n<4$, our approach improves rapidly as the increase of collaborative views. When $n \in [4,6]$, our approach achieves the best performance with DE$ \in [1.5121,1.5632]$ and uncertainty $\in [0.04,0.1]$. When $n>6$, the performance becomes stable with small fluctuations. 
Thus, integrating multi-view measurements can effectively improve the performance of collaborative localization.

\section{Conclusion}


Collaborative object localization is an essential capability for a multi-view system to collaboratively estimate the locations of objects. We propose a novel spatiotemporal graph filtering approach that integrates graph learning and model-based estimation in a principled fashion to perform multi-view sensor fusion for collaborative object localization. Our approach is able to model the complex relationship of objects and to fuse an arbitrary number of measurements to improve the performance. Extensive experiments are conducted to evaluate our approach in the scenarios of high-fidelity connected autonomous driving simulations and real-world multiple pedestrian localization. The experimental results show that our approach outperforms the previous methods and achieves the state-of-the-art performance on collaborative object localization.



\bibliographystyle{IEEEtran}
\bibliography{ref}

\begin{thebibliography}{10}
\providecommand{\url}[1]{#1}
\csname url@rmstyle\endcsname
\providecommand{\newblock}{\relax}
\providecommand{\bibinfo}[2]{#2}
\providecommand\BIBentrySTDinterwordspacing{\spaceskip=0pt\relax}
\providecommand\BIBentryALTinterwordstretchfactor{4}
\providecommand\BIBentryALTinterwordspacing{\spaceskip=\fontdimen2\font plus
\BIBentryALTinterwordstretchfactor\fontdimen3\font minus
  \fontdimen4\font\relax}
\providecommand\BIBforeignlanguage[2]{{%
\expandafter\ifx\csname l@#1\endcsname\relax
\typeout{** WARNING: IEEEtran.bst: No hyphenation pattern has been}%
\typeout{** loaded for the language `#1'. Using the pattern for}%
\typeout{** the default language instead.}%
\else
\language=\csname l@#1\endcsname
\fi
#2}}

\bibitem{chen2016monocular}
X.~Chen, K.~Kundu, Z.~Zhang, H.~Ma, S.~Fidler, and R.~Urtasun, ``Monocular {3D}
  object detection for autonomous driving,'' in \emph{IEEE Conference on
  Computer Vision and Pattern Recognition}, 2016.

\bibitem{qi2018frustum}
C.~R. Qi, W.~Liu, C.~Wu, H.~Su, and L.~J. Guibas, ``Frustum pointnets for {3D}
  object detection from {RGB-D} data,'' in \emph{IEEE conference on computer
  vision and pattern recognition}, 2018.

\bibitem{sindagi2019mvx}
V.~A. Sindagi, Y.~Zhou, and O.~Tuzel, ``{MVX-Net: Multimodal voxelnet for 3D
  object detection},'' in \emph{International Conference on Robotics and
  Automation}, 2019.

\bibitem{zhang2019robust}
W.~Zhang, H.~Zhou, S.~Sun, Z.~Wang, J.~Shi, and C.~C. Loy, ``Robust
  multi-modality multi-object tracking,'' in \emph{IEEE International
  Conference on Computer Vision}, 2019.

\bibitem{bowman2017probabilistic}
S.~L. Bowman, N.~Atanasov, K.~Daniilidis, and G.~J. Pappas, ``Probabilistic
  data association for semantic {SLAM},'' in \emph{IEEE international
  conference on robotics and automation}, 2017.

\bibitem{sharma2018beyond}
S.~Sharma, J.~A. Ansari, J.~K. Murthy, and K.~M. Krishna, ``Beyond pixels:
  {L}everaging geometry and shape cues for online multi-object tracking,'' in
  \emph{IEEE International Conference on Robotics and Automation}, 2018.

\bibitem{delmerico2018comparison}
J.~Delmerico, S.~Isler, R.~Sabzevari, and D.~Scaramuzza, ``A comparison of
  volumetric information gain metrics for active {3D} object reconstruction,''
  \emph{Autonomous Robots}, vol.~42, no.~2, pp. 197--208, 2018.

\bibitem{vasquez2014view}
J.~I. Vasquez-Gomez, L.~E. Sucar, and R.~Murrieta-Cid, ``View planning for {3D}
  object reconstruction with a mobile manipulator robot,'' in \emph{IEEE/RSJ
  International Conference on Intelligent Robots and Systems}, 2014.

\bibitem{wei2018survey}
S.~Wei, D.~Yu, C.~L. Guo, L.~Dan, and W.~W. Shu, ``Survey of connected
  automated vehicle perception mode: from autonomy to interaction,''
  \emph{Intelligent Transport Systems}, vol.~13, no.~3, pp. 495--505, 2018.

\bibitem{wasik2020robust}
A.~Wasik, P.~U. Lima, and A.~Martinoli, ``A robust localization system for
  multi-robot formations based on an extension of a {Gaussian} mixture
  probability hypothesis density filter,'' \emph{Autonomous Robots}, vol.~44,
  no.~3, pp. 395--414, 2020.

\bibitem{xu2016multi}
Y.~Xu, X.~Liu, Y.~Liu, and S.-C. Zhu, ``Multi-view people tracking via
  hierarchical trajectory composition,'' in \emph{IEEE Conference on Computer
  Vision and Pattern Recognition}, 2016.

\bibitem{guo2019collaborative}
R.~Guo, H.~Lu, P.~Gao, Z.~Zhang, and H.~Zhang, ``Collaborative localization for
  occluded objects in connected vehicular platform,'' in \emph{IEEE Vehicular
  Technology Conference}, 2019.

\bibitem{marvasti2020cooperative}
E.~E. Marvasti, A.~Raftari, A.~E. Marvasti, Y.~P. Fallah, R.~Guo, and H.~Lu,
  ``Cooperative {Lidar} object detection via feature sharing in deep
  networks,'' \emph{arXiv preprint}, 2020.

\bibitem{acevedo2020dynamic}
J.~J. Acevedo, J.~Messias, J.~Capit{\'a}n, R.~Ventura, L.~Merino, and P.~U.
  Lima, ``A dynamic weighted area assignment based on a particle filter for
  active cooperative perception,'' \emph{IEEE Robotics and Automation Letters},
  vol.~5, no.~2, pp. 736--743, 2020.

\bibitem{wang2018master}
H.~Wang, C.~Zhang, Y.~Song, and B.~Pang, ``Master-followed multiple robots
  cooperation {SLAM} adapted to search and rescue environment,''
  \emph{International Journal of Control, Automation and Systems}, 2018.

\bibitem{dogar2019multi}
M.~Dogar, A.~Spielberg, S.~Baker, and D.~Rus, ``Multi-robot grasp planning for
  sequential assembly operations,'' \emph{Autonomous Robots}, vol.~43, no.~3,
  pp. 649--664, 2019.

\bibitem{gao2020visual}
P.~Gao, B.~Reily, S.~Paul, and H.~Zhang, ``Visual reference of ambiguous
  objects for augmented reality-powered human-robot communication in a shared
  workspace,'' in \emph{International Conference on Human-Computer
  Interaction}, 2020.

\bibitem{Weng2020_AB3DMOT}
X.~Weng, J.~Wang, D.~Held, and K.~Kitani, ``{3D} multi-object tracking: {A}
  baseline and new evaluation metrics,'' \emph{International Conference on
  Intelligent Robots and Systems}, 2020.

\bibitem{wang2018constrained}
T.~Wang, H.~Ling, C.~Lang, S.~Feng, Y.~Jin, and Y.~Li, ``Constrained confidence
  matching for planar object tracking,'' in \emph{IEEE International Conference
  on Robotics and Automation}, 2018.

\bibitem{deng2019poserbpf}
X.~Deng, A.~Mousavian, Y.~Xiang, F.~Xia, T.~Bretl, and D.~Fox, ``{PoserBPF}:
  {A} rao-blackwellized particle filter for {6D} object pose tracking,''
  \emph{arXiv preprint}, 2019.

\bibitem{stenger2006model}
B.~Stenger, A.~Thayananthan, P.~H. Torr, and R.~Cipolla, ``Model-based hand
  tracking using a hierarchical bayesian filter,'' \emph{IEEE transactions on
  pattern analysis and machine intelligence}, vol.~28, no.~9, pp. 1372--1384,
  2006.

\bibitem{ullah2017hierarchical}
M.~Ullah, A.~K. Mohammed, F.~A. Cheikh, and Z.~Wang, ``A hierarchical feature
  model for multi-target tracking,'' in \emph{IEEE international conference on
  image processing}, 2017.

\bibitem{altche2017lstm}
F.~Altch{\'e} and A.~de~La~Fortelle, ``An {LSTM} network for highway trajectory
  prediction,'' in \emph{IEEE International Conference on Intelligent
  Transportation Systems}, 2017.

\bibitem{zhang2019sr}
P.~Zhang, W.~Ouyang, P.~Zhang, J.~Xue, and N.~Zheng, ``{SR-LSTM}: State
  refinement for lstm towards pedestrian trajectory prediction,'' in \emph{IEEE
  Conference on Computer Vision and Pattern Recognition}, 2019.

\bibitem{li2019grip}
X.~Li, X.~Ying, and M.~C. Chuah, ``{GRIP: G}raph-based interaction-aware
  trajectory prediction,'' in \emph{IEEE Intelligent Transportation Systems
  Conference}, 2019.

\bibitem{alahi2016social}
A.~Alahi, K.~Goel, V.~Ramanathan, A.~Robicquet, L.~Fei-Fei, and S.~Savarese,
  ``{Social-LSTM: H}uman trajectory prediction in crowded spaces,'' in
  \emph{IEEE conference on computer vision and pattern recognition}, 2016.

\bibitem{huang2019stgat}
Y.~Huang, H.~Bi, Z.~Li, T.~Mao, and Z.~Wang, ``{STGAT: Modeling}
  spatial-temporal interactions for human trajectory prediction,'' in
  \emph{IEEE International Conference on Computer Vision}, 2019.

\bibitem{ivanovic2019trajectron}
B.~Ivanovic and M.~Pavone, ``The trajectron: {P}robabilistic multi-agent
  trajectory modeling with dynamic spatiotemporal graphs,'' in \emph{IEEE
  International Conference on Computer Vision}, 2019.

\bibitem{yao2018mvsnet}
Y.~Yao, Z.~Luo, S.~Li, T.~Fang, and L.~Quan, ``{MvsNet: Depth} inference for
  unstructured multi-view stereo,'' in \emph{The European Conference on
  Computer Vision}, 2018.

\bibitem{li2018stereo}
P.~Li, T.~Qin, \emph{et~al.}, ``Stereo vision-based semantic {3D} object and
  ego-motion tracking for autonomous driving,'' in \emph{The European
  Conference on Computer Vision}, 2018.

\bibitem{ji2017surfacenet}
M.~Ji, J.~Gall, H.~Zheng, Y.~Liu, and L.~Fang, ``{SurfaceNet: An} end-to-end
  {3D} neural network for multiview stereopsis,'' in \emph{IEEE International
  Conference on Computer Vision}, 2017.

\bibitem{dong20174d}
J.~Dong, J.~G. Burnham, B.~Boots, G.~Rains, and F.~Dellaert, ``{4D} crop
  monitoring: Spatio-temporal reconstruction for agriculture,'' in \emph{IEEE
  International Conference on Robotics and Automation}, 2017.

\bibitem{wu2019accurate}
H.~Wu, X.~Zhang, B.~Story, and D.~Rajan, ``Accurate vehicle detection using
  multi-camera data fusion and machine learning,'' in \emph{IEEE International
  Conference on Acoustics, Speech and Signal Processing}, 2019.

\bibitem{brahmbhatt2018geometry}
S.~Brahmbhatt, J.~Gu, K.~Kim, J.~Hays, and J.~Kautz, ``Geometry-aware learning
  of maps for camera localization,'' in \emph{IEEE Conference on Computer
  Vision and Pattern Recognition}, 2018.

\bibitem{kendall2015posenet}
A.~Kendall, M.~Grimes, and R.~Cipolla, ``{PoseNet: A} convolutional network for
  real-time {6-DOF} camera relocalization,'' in \emph{IEEE international
  conference on computer vision}, 2015.

\bibitem{yin2018geonet}
Z.~Yin and J.~Shi, ``{GeoNet}: {Unsupervised} learning of dense depth, optical
  flow and camera pose,'' in \emph{IEEE Conference on Computer Vision and
  Pattern Recognition}, 2018.

\bibitem{roth2003real}
M.~Roth, D.~Vail, and M.~Veloso, ``A real-time world model for multi-robot
  teams with high-latency communication,'' in \emph{IEEE/RSJ International
  Conference on Intelligent Robots and Systems}, 2003.

\bibitem{li2018scene}
X.~Li, J.~Ylioinas, J.~Verbeek, and J.~Kannala, ``Scene coordinate regression
  with angle-based reprojection loss for camera relocalization,'' in \emph{The
  European Conference on Computer Vision}, 2018.

\bibitem{panzieri2006multirobot}
S.~Panzieri, F.~Pascucci, and R.~Setola, ``Multi-robot localisation using
  interlaced extended {Kalman} filter,'' in \emph{IEEE/RSJ International
  Conference on Intelligent Robots and Systems}, 2006.

\bibitem{qin2019surgical}
F.~Qin, Y.~Li, Y.-H. Su, D.~Xu, and B.~Hannaford, ``Surgical instrument
  segmentation for endoscopic vision with data fusion of rediction and
  kinematic pose,'' in \emph{International Conference on Robotics and
  Automation}, 2019.

\bibitem{meng2019signet}
Y.~Meng, Y.~Lu, A.~Raj, S.~Sunarjo, R.~Guo, T.~Javidi, G.~Bansal, and
  D.~Bharadia, ``{SIGNET: S}emantic instance aided unsupervised {3D} geometry
  perception,'' in \emph{IEEE Conference on Computer Vision and Pattern
  Recognition}, 2019.

\bibitem{chen2017multi}
X.~Chen, H.~Ma, J.~Wan, B.~Li, and T.~Xia, ``Multi-view {3D} object detection
  network for autonomous driving,'' in \emph{IEEE Conference on Computer Vision
  and Pattern Recognition}, 2017.

\bibitem{kleeman1992optimal}
L.~Kleeman \emph{et~al.}, ``Optimal estimation of position and heading for
  mobile robots using ultrasonic beacons and dead-reckoning.'' in
  \emph{International Conference on Robotics and Automation}, 1992.

\bibitem{ferryman2009pets2009}
J.~Ferryman and A.~Shahrokni, ``{Pets2009: Dataset} and challenge,'' in
  \emph{IEEE International workshop on performance evaluation of tracking and
  surveillance}, 2009.

\bibitem{gaoregularized}
P.~Gao, R.~Guo, H.~Lu, and H.~Zhang, ``Regularized graph matching for
  correspondence identification under uncertainty in collaborative
  perception,'' 2020.

\bibitem{godard2017unsupervised}
C.~Godard, O.~Mac~Aodha, and G.~J. Brostow, ``Unsupervised monocular depth
  estimation with left-right consistency,'' in \emph{IEEE Conference on
  Computer Vision and Pattern Recognition}, 2017.

\end{thebibliography}
\end{document}